\documentclass[11pt]{article}
\usepackage[hmargin=3.5cm,top=3cm,bottom=2.5cm,headsep=0.5cm,headheight=0.5cm,centering,includehead=false,includefoot=false,footskip=2cm,voffset=-1cm]{geometry}

\usepackage[english]{babel}

\usepackage{authblk}
\usepackage{titling}
\newcommand{\subtitle}[1]{%
	\posttitle{%
		\par\end{center}
	\begin{center}\sc#1\end{center}
	\vskip1em}}%

\usepackage{natbib}
\usepackage[colorlinks=False]{hyperref} 
\usepackage{url}            
\usepackage{booktabs}       
\usepackage[detect-all]{siunitx}
\usepackage{nicefrac}       
\usepackage{microtype}      
\usepackage[ruled,linesnumbered]{algorithm2e}
\usepackage{amssymb,amstext,amsmath,amsthm,amsfonts,pifont}
\usepackage{cleveref}
\usepackage{multirow}
\usepackage{tabularx,slashbox,multirow,diagbox}
\usepackage{mleftright}
\usepackage{xr-hyper}
\usepackage{adjustbox}

\usepackage{array}
\newcolumntype{L}[1]{>{\raggedright\let\newline\\\arraybackslash\hspace{0pt}}m{#1}}
\newcolumntype{C}[1]{>{\centering\let\newline\\\arraybackslash\hspace{0pt}}m{#1}}
\newcolumntype{R}[1]{>{\raggedleft\let\newline\\\arraybackslash\hspace{0pt}}m{#1}}

\usepackage{caption}
\usepackage{subcaption}
\usepackage{tikz,tikzscale,graphicx,epstopdf,tablefootnote}
\usetikzlibrary{arrows.meta,calc,decorations.markings,math,arrows.meta,shapes,arrows}
\tikzset{%
	base/.style = {rectangle, draw=black,
		minimum width=4cm, minimum height=1cm,
		text centered, font=\sffamily},
	binary/.style = {base, minimum width=1cm},
	startstop/.style = {base, fill=red!30, minimum width=2cm},
	activityRuns/.style = {base, fill=green!30},
	process/.style = {base, minimum width=2cm, fill=gray!15,
		font=\ttfamily},
	sum/.style      = {draw, circle, node distance = 1.5cm},
}
\usetikzlibrary{positioning}
\usetikzlibrary{shapes.misc}
\usetikzlibrary{automata,arrows,positioning,calc}

\theoremstyle{plain}
\newtheorem{theorem}{Theorem}[section]
\newtheorem{proposition}[theorem]{Proposition}

\newtheorem{definition}[theorem]{Definition}

\theoremstyle{remark}
\newtheorem{remark}{\textit{Remark}}
\newtheorem*{remark*}{\textit{Remark}}
\newtheorem{example}{\textit{Example}}
\newtheorem*{example*}{\textit{Example}}
\newtheorem{notation}{\textit{Notation}}
\newtheorem*{notation*}{\textit{Notation}}

\newcommand{\fref}[1]{Fig.~\ref{#1}}
\newcommand{\tref}[1]{Tab.~\ref{#1}}
\newcommand{\eref}[1]{(\ref{#1})}
\newcommand{\sref}[1]{Section~\ref{#1}}
\newcommand{\aref}[1]{Appendix~\ref{#1}}

\newcommand{\propref}[1]{Proposition~\ref{#1}}

\newcommand{\defref}[1]{Definition~\ref{#1}}

\newcommand{\True}{\mathrm{T}}
\newcommand{\False}{\mathrm{F}}
\newcommand{\xor}{\mathbf{xor}}
\newcommand{\xnor}{\mathbf{xnor}}
\newcommand{\andd}{\mathbf{and}}
\newcommand{\orr}{\mathbf{or}}

\newcommand{\bool}{\mathrm{bool}}

\def\Bb{\mathbf{B}}
\def\Bm{\mathrm{B}}
\def\Db{\mathbf{D}}
\def\Fc{\mathcal{F}}
\def\Lb{\mathbf{L}}
\def\Mb{\mathbf{M}}
\def\Nb{\mathbf{N}}

\def\Rb{\mathbf{R}}
\def\Sb{\mathbf{S}}

\def\Zb{\mathbf{Z}}

\def\xb{\mathbf{x}}

\newcommand{\pren}[1]{\mleft(#1\mright)}

\usepackage{xargs}

\DeclareMathOperator*{\deq}{\overset{{\rm def}}{=}}

\newcommand{\bvar}[1]{\delta#1} 
\newcommand{\one}[1]{\mathbf{1}\mleft(#1\mright)}

\def\sign{\operatorname{sign}}

\def\proj{\operatorname{p}}
\def\emb{\operatorname{e}}

\makeatletter

\def\env@sqcases{%
	\let\@ifnextchar\new@ifnextchar
	\left\lbrack
	\def\arraystretch{1.2}%
	\array{@{}l@{\quad}l@{}}%
}
\makeatother

\begin{document}


\title{Boolean Variation and Boolean Logic BackPropagation}
\subtitle{A Mathematical Principle of Boolean Logic Deep Learning}
\author{Van Minh Nguyen\\
	Intelligent Computing and Communication Research\\
	Advanced Wireless Technology Laboratory\\ 
	Paris Research Center, Huawei Technologies\\
	\texttt{vanminh.nguyen@huawei.com}
	}

\date{\today}
\maketitle


\begin{abstract}
	The notion of variation is introduced for the Boolean set and based on which Boolean logic backpropagation principle is developed. Using this concept, deep models can be built with weights and activations being Boolean numbers and operated with Boolean logic instead of real arithmetic. In particular, Boolean deep models can be trained directly in the Boolean domain without latent weights. No gradient but logic is synthesized and backpropagated through layers.
\end{abstract}

\section{Introduction}\label{sec:Intro}

Deep learning has become a commonplace solution to numerous modern problems, occupying a central spot of today's technological and social attention. The recipe of its power is the combining effect of unprecedented large dimensions and learning process based on gradient backpropagation \citep{LeCun1998}. In particular, thanks to the simplicity of neuron model that is decomposed into a weighted linear sum followed by a non-linear activation function, the gradient of weights is solely determined by their respective input without involving cross-parameter dependency. Hence, in terms of computational process, gradient backpropagation is automated by gradient chain rule and only requires a buffering of the forward input data.

However, deep learning is computationally intensive. \fref{fig:Intro_train_infer} shows its typical operations where the forward pass is used in both inference and training while the backpropagation is only for the training. In inference, the whole model parameters must be stored and the main computation is tensor-dot products. In training, the forward pass, in addition to that of the inference, needs to buffer all input tensors of every layer. They are used for tensor-dot products required by the derivative computation, gradient-descent-based optimizer, and gradient backpropagation. Also required by the gradient-based learning principle, model parameters and all signals are continuous numbers that are typically represented in 32-bit floating-point format.  
It results in large memory footprint. \fref{fig:Intro_Memories} shows our one example to illustrate the memory footprint of CIFAR-10 classification with VGG-small \citep{Simonyan2014}. This model requires 56~MB for parameters whereas its training with stochastic gradient descent and batch size of 100 requires 2~GB for training data and buffers, corresponding to 97\% of the total memory requirement. This is known as \emph{memory wall} problem \citep{Liao2021}. 
On the other hand, it is intensive in energy consumption due to large amount of floating-point multiplications. On top of that, not only one training process is repeated for a large number (millions) of forward-backpropagation iterations, but also the number of training experiments increases exponentially with the number of hyper-parameters for tuning. \citep{Strubell2019} showed that the CO2 footprint of training a natural language processing model is above 2000 times more than that of the inference. 
%
This problem has motivated a large body of literature that is summarized in \sref{sec:Related}. In particular, there has been motivations to develop \emph{binary neural networks} in which all parameters and/or data are binary numbers requiring only 1 bit. That would bring multiple benefits, notably significant reductions of memory footprint and of data movement and compute energy consumption. This line of research can be summarized into two approaches.

\begin{figure}[!t]
	\centering
	\includegraphics[width=0.85\columnwidth]{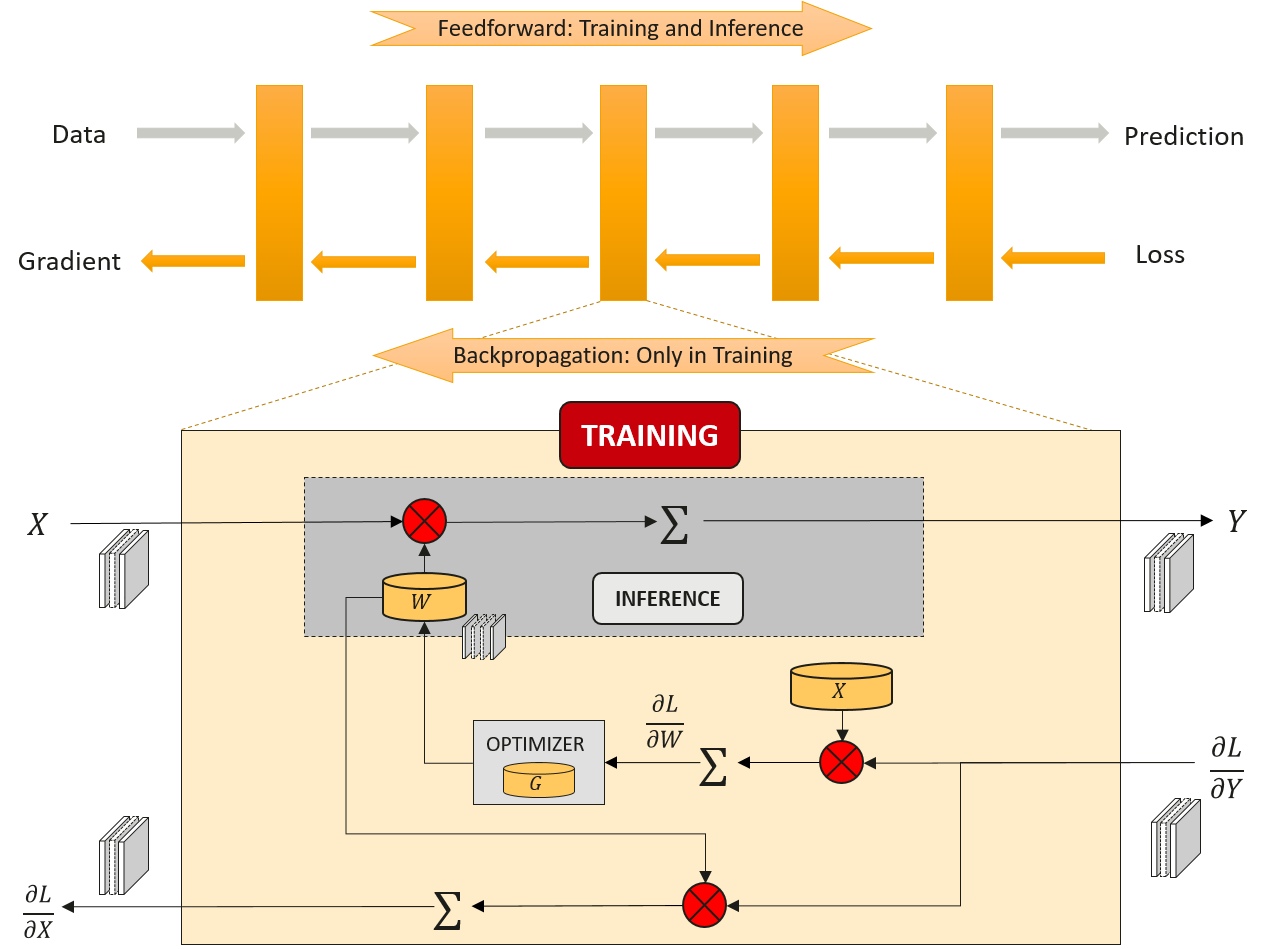}
	\caption{Overview of inference and training processes.}
	\label{fig:Intro_train_infer}
\end{figure}

\begin{figure}[!t]
	\centering
	\includegraphics[width=0.5\columnwidth]{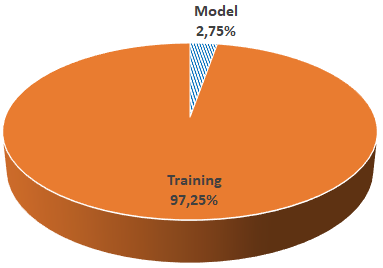}
	\caption[Example of training memory consumption]{Memory consumption of VGG-small model in 32-bit floating point and batch size of 100 samples for CIFAR10 classification.}
	\label{fig:Intro_Memories}
\end{figure}

The first one is called network binarization that is based on the existing gradient backpropagation for network training and aims at obtaining a binarized model for inference. This concept was proposed as BinaryConnect by \citep{Courbariaux2015}, Binarized neural networks (BNNs) by \citep{Hubara2016}, and XNOR-Net by \citep{Rastegari2016}. In BinaryConnect and BNNs, binary weights are obtained by sign extraction from the continuous ones during the continuous network training. In XNOR-Net, activations are real number and weights are binary. It also trains the continuous model by the standard gradient descent; then approximates the obtained continuous weights by binary weights with a scaling factor subjected to minimizing approximation error. Binarization approach in its core concept can perform relatively well in small tasks such as MNIST and CIFAR-10 classifications, but becomes hard to achieve comparable performance to the continuous model in more challenging tasks. Numerous architecture tweaks have been proposed, e.g., \cite{Bethge2019,Liu2018,Liu2020,zhang2022a,Guo2022}, aimed at bridging the gap to the continuous model in ImageNet classification. However, it is not clear how such architecture tweaks can be generalized to other computer vision tasks. On the other hand, network binarization is essentially a special case of quantization-aware training when the quantization function is the sign function. Its training is fully based on gradient backpropagation with continuous latent weights that are updated through the classic gradient descent. The derivative w.r.t. the quantized (or binarized) weights is not properly defined and needs to be approximated by a differential proxy. In \cite{Hubara2017}, the authors also showed that training a neural network with low-precision weights (even a single-layer one) is an NP-hard optimization problem. Due to weight discretization, the backpropagation algorithm used in continuous model training cannot be used effectively. The problem of unstable gradients appears as soon as the number of bits for representing numbers is less than 8. Through experiment study, \cite{Helwegen2019} observed that the role of latent weights in binarization process seems ignorable.

The second approach aims to train binary neural networks by another principle than the gradient backpropagation. 
To the be best of our knowledge, Expectation Backpropagation \citep{Soudry2014} was among the first attempts in this direction. It characterizes each weight by a probability distribution which is updated during the training. The network is then specified by sampling weights from the given distributions. Procedures along these lines operate on full-precision training. \citep{Baldassi2009} and \citep{Baldassi2015} proposed training algorithms inspired from statistical physics that keep and update an integer latent weight for every binary weight. It has been illustrated that these integer latent weights can also have bounded magnitude. This feature greatly reduces the complexity of operations. However, these algorithms are designed for specific types of networks and cannot be applied in a straightforward manner to DNNs and deep architectures. \citep{Baldassi2015a} provides an efficient algorithm based on belief propagation, but as in the previous works, this applies to a single perceptron and it is unclear how to use it for networks with multiple layers. An alternative approach to overcome the lack of gradient information is to address training as a combinatorial optimization problem and apply a known method for this kind of problems. Evolutionary Algorithms have been proposed as a training method for NNs in general by \citep{Morse2016}, and for low-precision DNNs by \citep{Ito2010}. However, evolutionary algorithms suffer from performance and scalability issues. First, although only low-precision weights are considered, the number of weights stored in memory is multiplied by the population size, and therefore, even for binary weights and a modest population size of 100, the memory footprint is larger than storing decimal weights with 16, 32 or 64 bits floating-point format. Second, a single forward pass during training has to be performed for all members of the population. Even though this can be done in parallel, having as many parallel processors as the population size is practically prohibitive. It is also well known that the population size should increase with the number of weights (dimension of the optimization parameter space), making scaling for big neural networks problematic.

The state of the art as described above shows that it appears possible to constraint deep neural networks to binary weights and/or activation, however how to reliably train binary deep models directly in binary domain remains an open question. This paper presents our attempt to provide a method to this question. Instead of formulating a kind of Boolean derivative such as one by \cite{Akers1959}, we introduce the notion of Boolean variation. Then, we develop a method by which one can synthesize logic rules for optimizing Boolean weights as well as ones to be backpropagated during the training phase.



\section{Related Works}\label{sec:Related}

Beside the most related works that are discussed in the introduction, this section provides a broader view of the literature that deals with the computational complexity of deep learning. It can be categorized into four approaches corresponding to four sources of complexity: large dimensions, intensive arithmetic, data bitwidth, and computing bottleneck.

Based on the fact that deep models are over-parameterized, low-rank factorization exploits this sparsity by means of tensor decomposition such as principal component analysis so as to only keep some most informative parts of the model. Several methods have been proposed in this direction known as compression and pruning \citep{Han2015,Cheng2020,Yang2017b}, or network design \citep{Sze2017,Howard2017,Tan2019}. Besides, knowledge distillation is also an important approach that trains a large model then uses it as a teacher to train a more compact one aimed at approximating or mimicking the function learned by the large model.

The second factor is arithmetic operations in which the most intensive are multiplications. Compared to additions, they are three time more expensive in 32-bit floating-point, and above 30 time more expensive in 32-bit integer \citep{Horowitz2014}. This factor becomes acute when considering large models. Arithmetic approximation proposes ways to tackle intensive multiplications \citep{Cong2014,Lav2020,Chen2020a}. For instance, several solutions exploit the computation structure of deep layers to reduce the number of multiplications, most notably Fast Fourier Transform, Winograd's algorithm transforms, and Strassen's algorithm-based approach \citep{Sze2020}.

Data bitwidth is the third factor that affects computational efficiency. Typically, all weights and signals are considered to be real numbers represented by 32-bit floating-point format. Quantization reduces the memory usage by reducing this number of bits to represent each parameter and activation, for instance to a 16-bit, or even 8-bit format, with acceptable degradation of the prediction accuracy. It has been observed \citep{Gholami2022} that quantization appears to be among the most efficient approaches regarding lower efforts and consistent gains it can achieve as compared to the other ones. Quantization methods can be classified into post-training quantization (i.e., quantized inference) \citep{Fei2022} that quantizes a full-precision trained model, quantization-aware training \citep{Gupta2015,Zhang2018,Jin2021,Yamamoto2021,Huang2021,Umuroglu2017}, and quantized training \citep{Chen2020,Sun2020,Yang2022a,Chmiel2021}. Compared to the continuous operation, cf. \fref{fig_Intro_FP_Process}, quantization-aware training, cf. \fref{fig_Intro_QuantInfer}, only quantizes the forward dataflows while keeping the training specific ones at full-precision. When the quantization function is the sign function, the resulting method becomes network binarization as described in the introduction. Quantized training in addition quantizes the backpropagated gradient. It however still keeps the weight gradient in full-precision for gradient descent optimizer, cf. \fref{fig_Intro_QuantTrain}.

\begin{figure}[!t]
	\centering
	\begin{subfigure}{.3\textwidth}
		\centering
		\includegraphics[width=.9\textwidth]{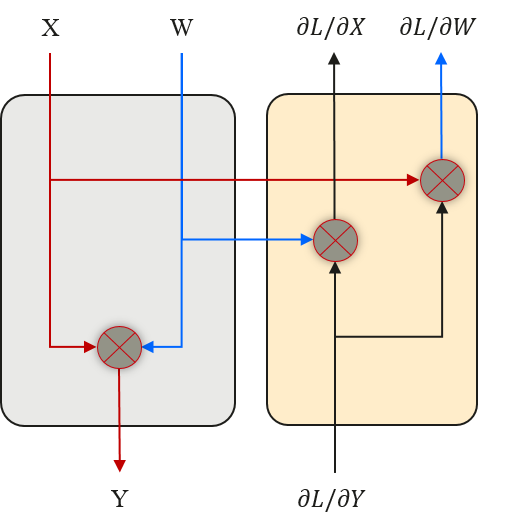}
		\subcaption{Full-precision.}
		\label{fig_Intro_FP_Process}
	\end{subfigure}%
	\begin{subfigure}{.3\textwidth}
		\centering
		\includegraphics[width=.9\textwidth]{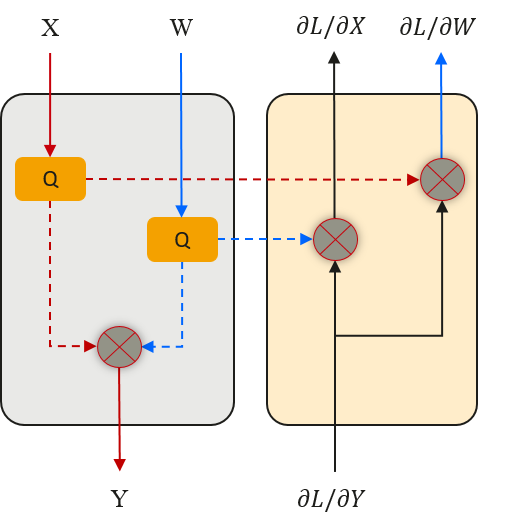}
		\subcaption{Quantization-aware training.}\label{fig_Intro_QuantInfer}
	\end{subfigure}%
	\begin{subfigure}{.3\textwidth}
		\centering
		\includegraphics[width=.9\textwidth]{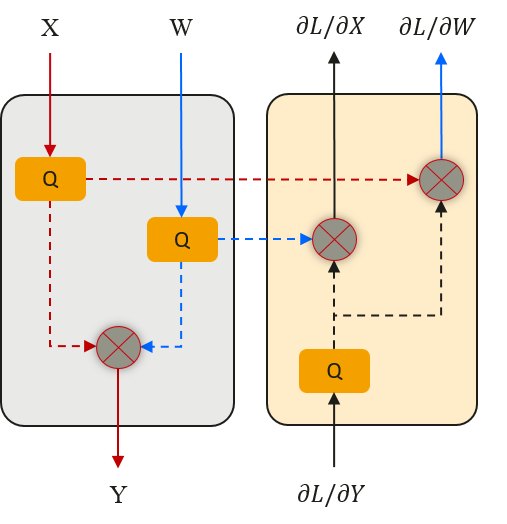}
		\subcaption{Quantized training.}
		\label{fig_Intro_QuantTrain}
	\end{subfigure}
	\caption{Data precision modalities during inference and training.}
	\label{fig:Intro_QuantMethods}
\end{figure}

\begin{figure}[!t]
	\centering
	\includegraphics[width=0.7\columnwidth]{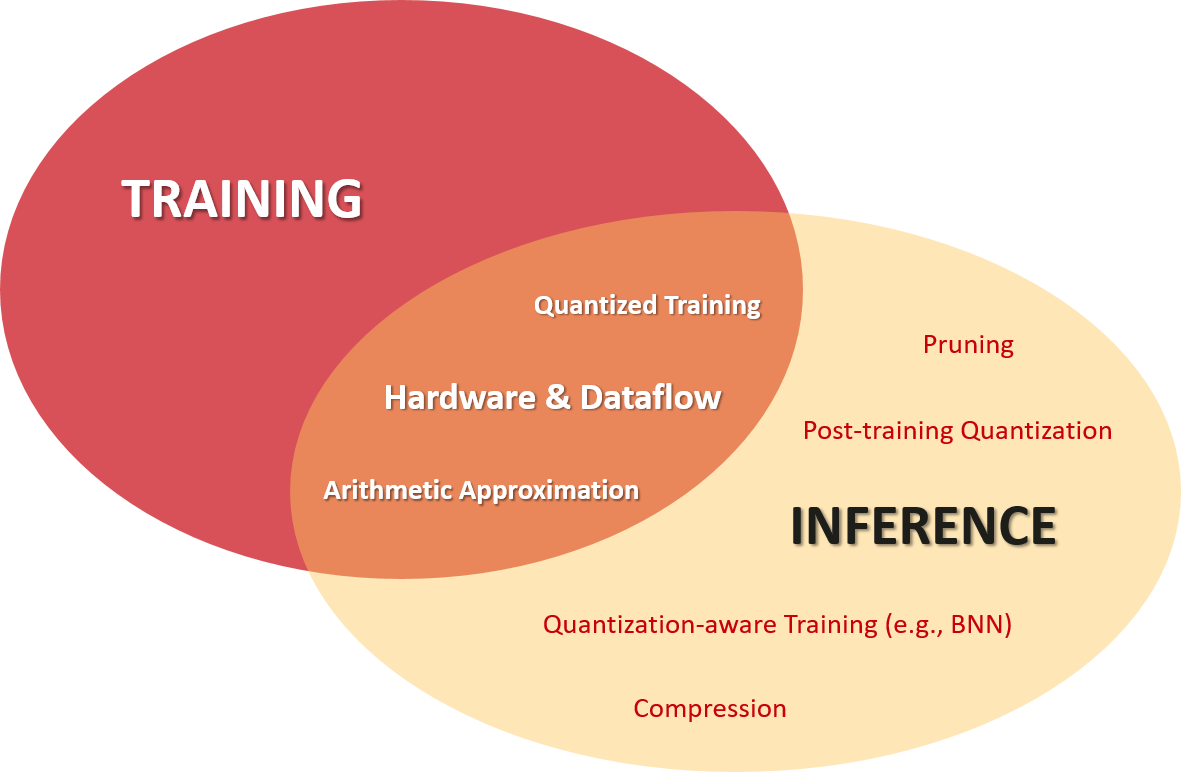}
	\caption{Overview of existing approaches and their benefits.}
	\label{fig:Intro_Sota}
\end{figure}

The above approaches have been mainly driven by the cost of arithmetic operations, i.e., aimed at reducing the number of floating-point operations (FLOPS) or per-operation cost. Even in the quantization approaches, although lowering data bitwidth is beneficial to both computation and data movement cost reduction (at the expense of performance loss), the design attention has been mainly attached to its effect on the computation. 
It has been revealed that the number of operations does not directly map to the actual computing hardware \citep{Sze2017,Sze2020a,Rutter2001,Yang2017a,Strubell2019}. Instead, the energy consumption is a prime measure of computing efficiency. Besides, memory footprint, latency, and silicon area are important factors of consideration. 
\citep{GarciaMartin2019} advocates that a reason of this lack of consideration by the machine learning community is due to possible unfamiliarity with energy consumption estimation as well as the lack of power models in the existing deep learning platforms such as TensorFlow, PyTorch, and MindSpore. Actually, energy consumption is due to not only computation but also data movement, which is often the dominant part. It is strictly tied to the system specification and architecture, requiring specific knowledge of computing systems and making it very hard to estimate or predict at a high level of precision. 
It also means that system architecture, memory hierarchy, data throughput, and dataflow constitute the forth factor of deep models' computational efficiency. 
On one hand, chip design has to trade off between the need of increasing near compute cache capacity, silicon area limitation, and energy consumption. In-memory or near-memory computing technologies \citep{Grimm2022,Verma2019} have been proposed that consist in performing the compute inside or nearest to the memory \citep{Sebastian2020,Alnatsheh2023}, hence reducing the data movement cost. New hardware technologies \citep{Yu2016,Yang2022} have been also investigated that increase the data movement efficiency. 
On the other hand, due to excessive large dimensions of deep models, off-chip memory remains mandatory. This comes the von Neumann bottleneck \citep{Zhang2022} that refers to an exponential increase of the energy cost of moving data between off-chip memory and the compute unit. In consequence, dataflow design \citep{Kwon2019} is a very important research area which aims at maximizing the data reuse so as to minimize the data movement cost. It is indeed a NP-hard problem \citep{Yang2020} in particular for deep layers that exhibit high reuse patterns such as convolution layers. Stationary dataflows have been a prominent strategy that proposes keeping one or several dataflows stationary in or near the compute unit. Most practical designs are input stationary, output stationary, weight stationary, and row stationary \citep{Chen2016a}.

Despite a lot of effort that has been given to deal with the computational intensiveness of deep models, except hardware and dataflow design that provides accelerated processing, most attention has been given to the inference complexity, for instance model compression, quantized inference, and quantization-aware training, cf. \fref{fig:Intro_Sota}.


\section{Design Insight and Method Introduction}\label{sec:Insight}

Boolean neuron is considered as a Boolean function. 
Besides the well-known non-linearity requirement, are there any other properties that Boolean neuron function should satisfy for the network learning capability? Let us take an example using Reed-Muller codes (RMC) \citep{MacWilliams1977} as the neuron function. For instance, as the $2^\textrm{nd}$ order RMC of size $m \geq 2$, the neuron is formulated as: $y = w_0 + \sum_{i=1}^m w_i x_i + \sum_{1 \leq i < j \leq m} w_{ij} x_i x_j$, where output $y$, weights $(w_i)$, $(w_{i,j})$, and inputs $(x_i)$ are binary, and the sums are in the binary field, i.e., modulo 2. No additional activation function is used since this model is non-linear of the inputs, and the weights are trained by the decoding mechanism provided by this code. Unfortunately, networks built on this model do not generalize in our experiments. Analyzing the reason behind led to an intuition that the neuron function instead should be \emph{elastic} in the sense that it should possess both responsive and absorptive regions in order to discriminate classes while being able to contract the points of the same class, which is not the case of this model. It implies that at some point inside the neuron's input-output transformation, the neuron needs to be embedded in a richer domain such as integer or real. 

Based on that understanding, the design proposed in the following is an effort to conciliate this intuition and the motivation of keeping the neuron in binary. Let $(w_0, w_1, w_2, \ldots)$ and $(x_1, x_2, \ldots)$ be Boolean numbers, and $\Bm$ be a Boolean logic operator such as AND, OR, XOR, XNOR which outputs 0 in place of False, and 1 in place of True. 
Define: 
\begin{equation}\label{eq:Preactivation}
	s = w_0 + \sum_{i} \Bm(w_i, x_i),
\end{equation}
where the sums are in the integer domain, which can be implemented in circuit using POPCOUNT. Let $\tau$ be a scalar, define:
\begin{equation}\label{eq:Activation}
	y = \begin{cases}
		\True,  & \textrm{if } s \geq \tau \\
		\False, & \textrm{if } s < \tau
	\end{cases}.
\end{equation}
This describes the neuron that we want to have for which $(w_i)$ and $(x_i)$ are Boolean weights and Boolean inputs, $\tau$ is threshold, 
and $s$ is the pre-activation. This is indeed a restricted case of the McCulloch–Pitts neuron in which the weights are now constrained to Boolean numbers. Specifically, it also takes advantage of the Logic Threshold Gate (LTG) proposed by \citep{McCulloch1943} which includes both summation and decision such that ``the decision can be made within the gate itself so that a pure binary output is produced'' \citep{Hampel1971}, hence no need of storing the integer pre-activation $s$ during the inference. LTG today has hardware implementation in standard CMOS logic \citep{Maan2017}. On the other hand, the step activation function used in \eref{eq:Activation} contains one responsive point and two absorptive regions at the two sides of the threshold. This model therefore satisfies our initial design target. 
Historically, McCulloch–Pitts neuron has been enriched until the use of differentiable activation function and continuous variables which allow the direct employment of the gradient descent training, leading to both the success and computational intensiveness of today's deep learning. Instead, we restrict it towards fully Boolean and propose in the following a method to train deep models of this Boolean neuron directly in the Boolean domain.

Instead of confronting a combinatorial problem of Boolean network training, we ask the question differently: as each weight only admits two values, should we keep or invert its current value? If we can construct a rule for making such a decision, the problem can be solved. Let us take a toy example. Consider a XOR neuron, i.e., $\mathrm{B} = \xor$. If we are given a signal $z$ that is back-propagated from the down-stream and indicates how the conditional/residual loss function at this stage varies with the neuron's forward output, then the rule can be constructed as follows. Suppose that the received signal $z$ indicates that the residual loss function increases with $s$. For weight $w_i$, we check its current value $\xor(w_i,x_i)$. If $\xor(w_i,x_i)$ is currently 1, it suffices to invert $w_i$ so that the new weight $w_i$ results in $\xor(w_i, x_i) = 0$ and makes $s$ decreased. Hence, the first part of the method is that: 
\begin{itemize}
	\item [Q1)] If given the backpropagation signal that expresses how the residual loss function varies w.r.t. the neuron's forward output, \emph{which weights to invert?}
\end{itemize}

To solve the above question, backpropagation signal $z$ needs to be given. In the case that neuron's downstream components provide the usual loss derivative, then $z$ is given and completely expresses the expected information. In the other case where neuron's downstream components do not have loss derivative, the second question to be solved is: 
\begin{itemize}
	\item[Q2)] \emph{What is the backpropagation signal of Boolean neurons?}
\end{itemize}

\if False
	For instance, when the loss derivative function, this condition is fulfilled. If the derivative is not available, we need to compute $z$, but let us skip it now.  $z = \True$ ($z = \False$, resp.) indicates that the loss increases (decreases, resp.) when $y$ changes from $\False$ to $\True$. 
	
	Suppose that $z = \True$, to decrease the loss we should make $y$ change from $\True$ to $\False$, meaning that we should decrease its pre-activation $s$. 
	
	Take for example $\mathrm{B} = \xor$. For weight $w_i$, we check its current value $\xor(w_i,x_i)$. If $\xor(w_i,x_i)$ is currently 1, it suffices to invert $w_i$ so that the new weight $w_i$ results in $\xor(w_i, x_i) = 0$. 
	The idea here is to establish variation relations of Boolean variables directly and synthesize logic rules to decide whether to invert Boolean weights or not and compute logic signal to be propagated to the up-stream.

	Related to the above question, there are cases for consideration. The first is where layer $(l+1)$ is a usual real-valued layer that back-propagates to layer $l$ the usual derivative of the loss w.r.t. the layer $l$ forward output. In this case, the backpropagation signal to layer $l$ satisfies the expected behavior and is given. The second case is where layer $(l+1)$ is a Boolean layer, its backpropagation signal is unknown, leading to the following second question to be solved: 
	\begin{itemize}
		\item[Q2)] How to compute the backpropagation signal of a Boolean neuron?
	\end{itemize}
\fi

\section{Boolean Variation}\label{sec:BoolVariation}

\subsection{Preliminary}\label{sec:Prelimimary}

Let $\Bb$ denote the set $\{\True, \False\}$ equipped with Boolean logic.

\begin{definition}\label{def:BoolOrder}
	Order relations `$<$' and `$>$' in $\Bb$ are defined as follows:
	\begin{equation}
		\False < \True, \quad \True > \False.
	\end{equation}
\end{definition}

\begin{definition}\label{def:BoolVariation}
	For $x, y \in \Bb$, define:
	\begin{equation}
		\bvar(x \to y) \deq \begin{cases} 
			\True, & \textrm{if } y > x,\\
			0, & \textrm{if } y = x,\\
			\False, & \textrm{if } y < x,
		\end{cases}
	\end{equation}
	which is called the variation of $x$ to $y$.
\end{definition}

\begin{definition}[Three-valued logic]\label{def:ThreeValueLogic}
	Define $\Mb = \Bb \cup \{0\}$ with logic connectives specified in \tref{tab:ThreeValueLogic}.
	{\renewcommand{\arraystretch}{1.2}
		\begin{table}[!h]
			\scalebox{0.8}{
			\begin{subtable}[c]{0.15\textwidth}
				\centering
				\vspace{0.5cm}
				\begin{tabular}{| c | c |}
					\hline
					$a$  	&  	$\neg a$ 	\\ \hline 
					$\True$	& 	$ \False$ 	\\ \hline 
					$0$  		& $0$ 		\\ \hline 
					$\False$	& $\True$ 	\\ \hline
				\end{tabular}
				\subcaption{Negation}
			\end{subtable}}%
			\if False
			\begin{subtable}[c]{\textwidth}
				\centering
				\begin{tabular}{| c | c |}
					\hline
					$a$  	&  	$\neg a$ 	\\ \hline 
					$\True$	& 	$ \False$ 	\\ \hline 
					$0$  		& $0$ 		\\ \hline 
					$\False$	& $\True$ 	\\ \hline
				\end{tabular}
				\subcaption{Negation}
			\end{subtable}
			\fi
			\scalebox{0.8}{
			\begin{subtable}{0.25\textwidth}
				\centering
				\vspace{0.5cm}
				\begin{tabular}{| c | c | c | c | c |}
					\hline
					\multicolumn{2}{|c|}{\multirow{2}{*}{$\andd$}} & \multicolumn{3}{c|}{$b$} 	\\ \cline{3-5} 
					\multicolumn{2}{|c|}{}				& $\True$	& $0$	& $\False$ 	\\ \hline
					\multirow{3}{*}{$a$}	& $\True$	& $\True$ 	& $0$	& $\False$	\\ \cline{2-5}
					& $0$		& $0$  		& $0$	& $0$		\\ \cline{2-5}
					& $\False$	& $\False$  & $0$	& $\False$	\\ \hline
				\end{tabular}
				\subcaption{AND}
			\end{subtable}}%
			\scalebox{0.8}{
			\begin{subtable}{0.25\textwidth}
				\centering
				\vspace{0.5cm}
				\begin{tabular}{| c | c | c | c | c |}
					\hline
					\multicolumn{2}{|c|}{\multirow{2}{*}{$\orr$}} & \multicolumn{3}{c|}{$b$} 	\\ \cline{3-5} 
					\multicolumn{2}{|c|}{}				& $\True$	& $0$	& $\False$ 	\\ \hline
					\multirow{3}{*}{$a$}	& $\True$	& $\True$ 	& $0$	& $\True$	\\ \cline{2-5}
					& $0$		& $0$  		& $0$	& $0$		\\ \cline{2-5}
					& $\False$	& $\True$  	& $0$	& $\False$	\\ \hline
				\end{tabular}
				\subcaption{OR}			
			\end{subtable}}%
			\scalebox{0.8}{
				\begin{subtable}{0.25\textwidth}
				\centering
				\vspace{0.5cm}
				\begin{tabular}{| c | c | c | c | c |}
					\hline
					\multicolumn{2}{|c|}{\multirow{2}{*}{$\xor$}} & \multicolumn{3}{c|}{$b$} 	\\ \cline{3-5} 
					\multicolumn{2}{|c|}{}				& $\True$	& $0$	& $\False$ 	\\ \hline
					\multirow{3}{*}{$a$}	& $\True$	& $\False$ 	& $0$	& $\True$	\\ \cline{2-5}
					& $0$		& $0$  		& $0$	& $0$		\\ \cline{2-5}
					& $\False$	& $\True$  	& $0$	& $\False$	\\ \hline
				\end{tabular}
				\subcaption{XOR}			
			\end{subtable}}%
			\scalebox{0.8}{
			\begin{subtable}{0.25\textwidth}
				\centering
				\vspace{0.5cm}
				\begin{tabular}{| c | c | c | c | c |}
					\hline
					\multicolumn{2}{|c|}{\multirow{2}{*}{$\xnor$}} & \multicolumn{3}{c|}{$b$} 	\\ \cline{3-5} 
					\multicolumn{2}{|c|}{}				& $\True$	& $0$	& $\False$ 	\\ \hline
					\multirow{3}{*}{$a$}	& $\True$	& $\True$ 	& $0$	& $\False$	\\ \cline{2-5}
					& $0$		& $0$  		& $0$	& $0$		\\ \cline{2-5}
					& $\False$	& $\False$ 	& $0$	& $\True$	\\ \hline
				\end{tabular}
				\subcaption{XNOR}
			\end{subtable}}
			\caption{Logic connectives of the three-value logic $\Mb$.}
			\label{tab:ThreeValueLogic}
	\end{table}}
\end{definition}

$\Mb$ is an extension of the Boolean logic by adding element $0$ whose meaning is ``ignored'', and the connection of its logic connectives to those of Boolean logic can be expressed as follows. Let $l$ be a logic connective, denote by $\l_{\Mb}$ and $\l_{\Bb}$ when it is applied in $\Mb$ and in $\Bb$, respectively. Then, what defined \defref{def:ThreeValueLogic} can be expressed as follows:
\begin{equation}
	l_{\Mb}(a,b) = \begin{cases}
		0, & \textrm{if } a = 0 \textrm{ or } b = 0,\\
		l_{\Bb}(a,b), & \textrm{otherwise}.
	\end{cases}
\end{equation}

In the sequel, the following notations are used:
\begin{itemize}
	\item $\Rb$: the real set; $\Db$: a discrete set; $\Nb$: a numeric set, e.g., $\Rb$ or $\Db$; 
	\item $\Lb$: a logic set, e.g., $\Bb$ or $\Mb$.
	\item $\Sb$: a certain domain that can be any of $\Lb$ or $\Nb$.
\end{itemize}

\begin{definition}
	The magnitude of a variable $x$, denoted $|x|$, is defined as follows:
	\begin{itemize}
		\item $x \in \Bb$: $|x| = 1$.
		\item $x \in \Mb$: $|x| = 0$ if $x = 0$, and $|x| = 1$ otherwise.
		\item $x \in \Nb$: $|x|$ is defined as its usual absolute value.
	\end{itemize}
\end{definition}

\begin{definition}[Type conversion]\label{def:Conversion}
	Define:
	\begin{align}
		\proj \colon & \Nb \to \Lb \nonumber\\
		& x \mapsto x_L = \proj(x) = \begin{cases}
			\True, & \textrm{if } x > 0,\\
			0, & \textrm{if } x = 0, \\
			\False, & \textrm{if } x < 0.
		\end{cases}\label{eq:Projector}
	\end{align}
	\begin{align}
		\emb \colon & \Lb \to \Nb \nonumber\\
		& a \mapsto a_N = \emb(a) = \begin{cases}
			+1, & \textrm{if } a = \True,\\
			0, & \textrm{if } a = 0, \\
			-1, & \textrm{if } a = \False.
		\end{cases}\label{eq:Embedding}
	\end{align}
\end{definition}
$\proj$ projects a numeric type in logic, and $\emb$ embeds a logic type in numeric. It is clear that for $x, y \in \Nb$: 
\begin{equation}
	x = y \Leftrightarrow |x| = |y| \textrm{ and } \proj(x) = \proj(y).
\end{equation}
\if False
	\begin{definition}
		The logic value of a variable $x$, denoted $x_L$, is defined as follows:
		\begin{itemize}
			\item For $x \in \Lb$: $x_L = x$.
			\item For $x \in \Nb$: $x_L \in \Mb$ st. $x_L = \True$ if $x > 0$, $x_L = 0$ if $x = 0$, and $x_L = \False$ if $x < 0$.
		\end{itemize}
	\end{definition}
\fi

\begin{definition}[Mixed-type logic]\label{def:MixedLogic}
	For $l$ a logic connective of $\Lb$ and two variables $a$, $b$, operation $c = l(a, b)$ is defined such that $|c| = |a||b|$ and $c_L = l(a_L, b_L)$. 
\end{definition}

\begin{proposition}\label{prop:XNORAlgebra}
	The following properties hold:
	\begin{enumerate}
		\item $a \in \Lb$, $x \in \Nb$: 
		\begin{equation}
			\xnor(a,x) = \begin{cases} x, & \textrm{if } a = \True,\\
				0, & \textrm{if } a = 0,\\
				-x, & \textrm{if } a = \False.
				\end{cases}
		\end{equation}
		\item $x, y \in \Nb$: $\xnor(x,y) = xy$. 
		\item $x \in \{\Lb, \Nb\}$, $y, z \in \Nb$: $\xnor(x,y+z) = \xnor(x,y) + \xnor(x,z)$. 
		\item $x \in \{\Lb, \Nb\}$, $y, \lambda \in \Nb$: $\xnor(x, \lambda y) = \lambda \xnor(x,y)$.
		\item $x \in \{\Lb, \Nb\}$, $y \in \Nb$: $\xor(x,y) = -\xnor(x,y)$.
	\end{enumerate}
\end{proposition}
\begin{proof} 
	The proof follows definitions \ref{def:MixedLogic} and \ref{def:Conversion}.
	\begin{itemize}
		\item Following \defref{def:ThreeValueLogic} we have $\forall t \in \Mb$, $\xnor(\True, t) = t$, $\xnor(\False, t) = \neg t$, and $\xnor(0, t) = 0$. Put $v = \xnor(a, x)$. We have $|v| = |x|$ and $v_L = \xnor(a, x_L)$. Hence, $a = 0 \Rightarrow v_L = 0 \Rightarrow v = 0$; $a = \True \Rightarrow v_L = x_L \Rightarrow v = x$; $a = \False \Rightarrow v_L = \neg x_L \Rightarrow v = -x$. Hence (1).
		\item The result is trivial if $x=0$ or $y=0$. For $x, y \neq 0$, put $v = \xnor(x,y)$, we have $|v| = |x||y|$ and $v_L = \xnor(x_L, y_L)$. According to \defref{def:Conversion}, if $\sign(x)=\sign(y)$, we have $v_L = \True \Rightarrow v = |x||y| = xy$. Otherwise, i.e., $\sign(x)=-\sign(y)$, $v_L = \False \Rightarrow v = -|x||y| = xy$. Hence (2).
		\item (3) and (4) follow (1) for $x \in \Lb$ and follow (2) for $x \in \Nb$.
		\item For (5), write $u = \xor(x,y)$ and $v = \xnor(x,y)$, we have $|u| = |v|$ and $u_L = \xor(x_L, y_L) = \neg \xnor(x_L, y_L) = \neg v_L$. Thus, $\sign(u) = -\sign(v) \Rightarrow u = -v$. \qedhere
	\end{itemize}
\end{proof}

\subsection{Variation Calculus}

\begin{definition}\label{def:BoolFuncVar}
	For $f \in \Fc(\Bb, \Bb)$, the variation of $f$ w.r.t. its variable, denoted $f'$, is defined as follows:
	\begin{align*}
		f' \colon & \Bb \to \Mb \\
		& x \mapsto \xnor(\bvar(x \to \neg x), \bvar f(x \to \neg x)). 
	\end{align*}
\end{definition}
Intuitively, the variation of $f$ is True if $f$ varies in the same direction with $x$.

\begin{notation}
	We use notation $\bvar f(x)/ \bvar x$ to denote the variation of a function $f$ w.r.t. a variable $x$, i.e.,
	\begin{equation}
		\frac{\bvar f(x)}{\bvar x} := f'(x).
	\end{equation}
\end{notation}

\begin{example}
	For $f = \xor(a, b)$ with $a, b \in \Bb$, the variation of $f$ w.r.t. each variable according to \defref{def:BoolFuncVar} can be derived by establishing a truth table as in \tref{tab:VariationXOR} from which we obtain:
	\begin{equation}
		f'_a(b) = \neg a.
	\end{equation}
	{\renewcommand{\arraystretch}{1.0}
		\begin{table}[!h]
			\centering
			\begin{tabular}{cccccccc}
				\toprule 
				\multirow{2}{*}{$a$} 
				& \multirow{2}{*}{$b$} 
				& \multirow{2}{*}{$\neg b$} 
				& \multirow{2}{*}{$\bvar(b \to \neg b)$} 
				& \multirow{2}{*}{$f(a,b)$}
				& \multirow{2}{*}{$f(a,\neg b)$}
				& \multirow{2}{*}{$\bvar f_a(b \to \neg b)$} 
				& \multirow{2}{*}{$f'_a(b)$}   \\
				& & & & & \\
				\midrule 
				$\True$  &  $\True$  &  $\False$  & $\False$  &  $\False$ 	& $\True$	& $\True$  & $\False$  \\
				
				$\True$  &  $\False$ &  $\True$   & $\True$   &  $\True$	& $\False$ 	& $\False$ & $\False$  \\
				
				$\False$ &  $\True$  &  $\False$  & $\False$  &  $\True$	& $\False$ 	& $\False$ & $\True$   \\
				
				$\False$ &  $\False$ &  $\True$   & $\True$   &  $\False$ 	& $\True$ 	& $\True$  & $\True$   \\
				\bottomrule
			\end{tabular}
			\caption{Variation truth table of $\xor$.}
			\label{tab:VariationXOR}
		\end{table}
	}
\end{example}

\begin{proposition}\label{prop:B2BVariation}
	For $f, g \in \Fc(\Bb, \Bb)$, $\forall x, y \in \Bb$ the following properties hold:
	\begin{enumerate}
		\item $\bvar f(x \to y) = \xnor(\bvar(x \to y), f'(x)).$
		\item $(\neg f)'(x) = \neg f'(x)$.
		\item $(g \circ f)'(x) = \xnor(g'(f(x)), f'(x))$.
	\end{enumerate}
\end{proposition}
\begin{proof} The proof is by definition:
	\begin{enumerate}
		\item $\forall x, y \in \Bb$, there are two cases. If $y = x$, then the result is trivial. Otherwise, i.e., $y = \neg x$, by definition we have:
		\begin{align*}
			f'(x) & = \xnor(\bvar(x \to \neg x), \bvar f(x \to \neg x)) \\
			\Leftrightarrow \quad \bvar f(x \to \neg x) & = \xnor(\bvar(x \to \neg x), f'(x)).
		\end{align*}
		Hence the result.
		\item $\forall x, y \in \Bb$, it is easy to verify by truth table that $\bvar(\neg f(x \to y)) = \neg \bvar{f(x \to y)}$. Hence, by definition,
		\begin{align*}
			(\neg f)'(x) & = \xnor(\bvar(x \to \neg x), \bvar(\neg f(x \to \neg x))) \\
			& = \xnor(\bvar(x \to \neg x), \neg \bvar f(x \to \neg x)) \\
			& = \neg \xnor(\bvar(x \to \neg x), \bvar f(x \to \neg x)) \\
			& = \neg f'(x).
		\end{align*}
		\item Using definition, property (i), and associativity of $\xnor$, $\forall x \in \Bb$ we have:
		\begin{align*}
			(g\circ f)'(x) & = \xnor(\bvar(x \to \neg x), \bvar g(f(x) \to f(\neg x))) \\
			& = \xnor\pren{\bvar(x \to \neg x), \xnor\pren{\bvar f(x \to \neg x), g'\pren{f(x)} }} \\
			& = \xnor\pren{g'(f(x)), \xnor\pren{\bvar(x \to \neg x), \bvar f(x \to \neg x) } } \\
			& = \xnor(g'(f(x)), f'(x) ).
		\end{align*} \qedhere
	\end{enumerate}
\end{proof}

\begin{definition}
	For $f \in \Fc(\Bb, \Nb)$, the variation of $f$ w.r.t. its variable, denoted $f'$, is defined as follows:
	\begin{align*}
		f' \colon & \Bb \to \Nb \\
		& x \mapsto \xnor(\bvar(x \to \neg x), \bvar f(x \to \neg x)).
	\end{align*}
\end{definition}

\begin{proposition}\label{prop:B2NVariation}
	For $f \in \Fc(\Bb, \Nb)$, the following properties hold:
	\begin{enumerate}
		\item $x, y \in \Bb$: $\bvar f(x \to y) = \xnor(\bvar(x \to y), f'(x))$.
		\item $x \in \Bb$, $\alpha \in \Nb$: $(\alpha f)'(x) = \alpha f'(x)$.
		\item $x \in \Bb$, $g \in \Fc(\Bb, \Nb)$: $(f + g)'(x) = f'(x) + g'(x)$.
	\end{enumerate}
\end{proposition}
\begin{proof}
	The proof is as follows:
	\begin{enumerate}
	\item For $x, y \in \Bb$. Firstly, the result is trivial if $y = x$. For $y \neq x$, i.e., $y = \neg x$, by definition:
		\begin{equation*}
			f'(x) = \xnor(\bvar(x \to \neg x), \bvar f(x \to \neg x)).
		\end{equation*}
		Hence, $|\bvar f(x \to \neg x)| = |f'(x)|$ since $|\bvar(x \to \neg x)| = 1$, and
		\begin{align*}
			\proj(f'(x)) & = \xnor(\bvar(x \to \neg x), \proj(\bvar f(x \to \neg x)))\\
			\Leftrightarrow \quad \proj(\bvar f(x \to \neg x)) & = \xnor(\bvar(x \to \neg x), \proj(f'(x))),
		\end{align*}
		where $p$ is the logic projector \eref{eq:Projector}. Thus, $f(x \to \neg x) = \xnor(\bvar(x \to \neg x), f'(x))$. Hence the result.
	\item Firstly $\forall x, y \in \Bb$, we have
		\begin{equation*}
			\bvar(\alpha f(x \to y)) = \alpha f(y) - \alpha f(x) = \alpha \bvar{f(x \to y)}.
		\end{equation*}
		Hence, by definition,
		\begin{align*}
			(\alpha f)'(x) & = \xnor(\bvar(x \to \neg x), \bvar(\alpha f(x \to \neg x))) \\
			& = \xnor(\bvar(x \to \neg x), \alpha\bvar{f(x \to \neg x)}) \\
			& = \alpha \, \xnor(\bvar(x \to \neg x), \bvar{f(x \to \neg x)}), \textrm{ due to \propref{prop:XNORAlgebra}(4)}\\
			& = \alpha f'(x).
		\end{align*}
	\item For $f, g \in \Fc(\Bb, \Nb)$,
		\begin{align*}
			(f+g)'(x) & = \xnor(\bvar(x \to \neg x), \bvar(f+g)(x \to \neg x))\\
			& = \xnor(\bvar(x \to \neg x), \bvar f(x \to \neg x) + \bvar g(x \to \neg x)) \\
			& \overset{(*)}{=} \xnor(\bvar(x \to \neg x), \bvar f(x \to \neg x)) + \xnor(\bvar(x \to \neg x), \bvar g(x \to \neg x)),\\
			& = f'(x) + g'(x),
		\end{align*}
		where $(*)$ is due to \propref{prop:XNORAlgebra}(3). \qedhere
	\end{enumerate}
\end{proof}

For discrete-variable functions, recall that the discrete derivative of $f \in \Fc(\Db, \Nb)$ has been defined in the literature as $f'(n) = f(n+1) - f(n)$. With the logic variation as previously introduced in this section, we can have a more generic definition of the variation of a discrete-variable function as follows. 
\begin{definition}
	For $f \in \Fc(\Db, \Sb)$, the variation of $f$ w.r.t. its variable is defined as follows:
	\begin{align*}
		f' \colon & \Db \to \Sb \\
		& x \mapsto \bvar f(x \to x+1),
	\end{align*}
	where $\bvar f$ is in the sense of the variation defined in $\Sb$.
\end{definition}

\begin{proposition}[Composition rules]\label{prop:Composition}
	The following properties hold:
	\begin{enumerate}
		\item For $\Bb \overset{f}{\to} \Bb \overset{g}{\to} \Sb$: $(g \circ f)'(x) = \xnor(g'(f(x)), f'(x))$, $\forall x \in \Bb$.
		\item For $\Bb \overset{f}{\to} \Db \overset{g}{\to} \Sb$, $x \in \Bb$, if $|f'(x)| \leq 1$ and $g'(f(x)) =g'(f(x)-1)$, then:
		\begin{equation*}
			(g \circ f)'(x) = \xnor(g'(f(x)), f'(x)).
		\end{equation*}
	\end{enumerate}
\end{proposition}
\begin{proof}
	The proof is as follows.
	\begin{enumerate}
		\item The case of $\Bb \overset{f}{\to} \Bb \overset{g}{\to} \Bb$ is obtained from \propref{prop:B2BVariation}(3). For $\Bb \overset{f}{\to} \Bb \overset{g}{\to} \Nb$, by using \propref{prop:B2NVariation}(1), the proof is similar to that of \propref{prop:B2BVariation}(3).
		\item By definition, we have
		\begin{equation}\label{eq:proofComposition2_1}
			(g \circ f)'(x) = \xnor(\bvar(x \to \neg x), \bvar g(f(x) \to f(\neg x))). 
		\end{equation} 
		Using property (1) of \propref{prop:B2NVariation}, we have:
		\begin{align}
			f(\neg x) & = f(x) + \bvar f(x \to \neg x) \nonumber\\
			& = f(x) + \xnor(\bvar(x \to \neg x), f'(x)). \label{eq:proofComposition2_2}
		\end{align}
		Applying \eref{eq:proofComposition2_2} back to \eref{eq:proofComposition2_1}, the result is trivial if $f'(x) = 0$. The remaining case is $|f'(x)| = 1$ for which we have $\xnor(\bvar(x \to \neg x), f'(x)) = \pm 1$. First, for $\xnor(\bvar(x \to \neg x), f'(x)) = 1$, we have:
		\begin{align}
			\bvar g(f(x) \to f(\neg x)) &= \bvar g(f(x) \to f(x) + 1) \nonumber\\
			& = g'(f(x)) \nonumber\\
			& = \xnor(g'(f(x)), 1) \nonumber\\
			& = \xnor(g'(f(x)), \xnor(\bvar(x \to \neg x), f'(x)) ) \label{eq:proofComposition2_3}.
		\end{align}
		Substitute \eref{eq:proofComposition2_3} back to \eref{eq:proofComposition2_1}, we obtain:
		\begin{align*}
			(g \circ f)'(x) & = \xnor(\bvar(x \to \neg x), \bvar g(f(x) \to f(\neg x))) \\ 
			& = \xnor(\bvar(x \to \neg x), \xnor(g'(f(x)), \xnor(\bvar(x \to \neg x), f'(x)) ) ) \\
			& = \xnor(g'(f(x)), f'(x)),
		\end{align*} 
		where that last equality is by the associativity of $\xnor$ and that $\xnor(x, x) = \True$ for $x \in \Bb$. 
		Similarly, for $\xnor(\bvar(x \to \neg x), f'(x)) = -1$, we have:
		\begin{align}
			\bvar g(f(x) \to f(\neg x)) &= \bvar g(f(x) \to f(x) - 1) \nonumber\\
			& = - g'(f(x)-1) \nonumber\\
			& = \xnor(g'(f(x)-1), -1) \nonumber\\
			& = \xnor(g'(f(x)-1), \xnor(\bvar(x \to \neg x), f'(x)) ) \label{eq:proofComposition2_4}.
		\end{align}
		Substitute \eref{eq:proofComposition2_4} back to \eref{eq:proofComposition2_1} and use the assumption that $g'(f(x)) = g'(f(x)-1)$, we have:
		\begin{align*}
			(g \circ f)'(x) & = \xnor(\bvar(x \to \neg x), \bvar g(f(x) \to f(\neg x))) \\ 
			& = \xnor(\bvar(x \to \neg x), \xnor(g'(f(x)-1), \xnor(\bvar(x \to \neg x), f'(x)) ) ) \\
			& = \xnor(g'(f(x)), f'(x)).
		\end{align*} 
		Hence the result. \qedhere
	\end{enumerate}
\end{proof}

The above formulation is extended to the multivariate case in a straightforward manner. In the following, we show one example with Boolean functions. 

\begin{definition}\label{def:BoolFuncVar0}
	For $\xb = (x_1, \ldots, x_n) \in \Bb^n$, $n \ge 1$, and $1 \leq i \leq n$, denote: 
	\begin{equation*}
		\xb_{\neg i} := (x_1, \ldots, x_{i-1}, \neg x_i, x_{i+1}, \ldots, x_n).
	\end{equation*}
	For $f \in \Fc(\Bb^n, \Bb)$, the (partial) variation of $f$ w.r.t. $x_i$, denoted $f'_{i}(\xb)$ and $\frac{\bvar f(\xb)}{\bvar x_i}$, is defined as follows:
	\begin{equation}
		f'_{i}(\xb) = \frac{\bvar f(\xb)}{\bvar x_i} \deq \xnor(\bvar(x_i \to \neg x_i), \bvar f(\xb \to \xb_{\neg i})).
	\end{equation}
\end{definition}

It is easy so show that $f'_{i}$ has the same properties as in \propref{prop:B2BVariation}. For instance, the composition is as follows.

\begin{proposition}
	For $f \in \Fc(\Bb^n, \Bb)$, $g \in \Fc(\Bb, \Bb)$, $1 \le i \le n$, we have:
	\begin{equation}
		(g \circ f)'_i(\xb) = \xnor(g'(f(\xb)), f'_i(\xb)).
	\end{equation}
\end{proposition} 
\begin{proof}
	By definition:
	\begin{align*}
		(g \circ f)'_i(\xb) & = \xnor\left(\bvar(x_i \to \neg x_i), \bvar g(f(\xb) \to f(\xb_{\neg i}))\right)\\
		& = \xnor\left(\bvar(x_i \to \neg x_i), \xnor\left(\bvar f(\xb \to \xb_{\neg i}), g'(f(\xb)) \right) \right), \; \textrm{\propref{prop:B2BVariation}(1)}\\\
		& = \xnor\left(\xnor(\bvar(x_i \to \neg x_i), \bvar f(\xb \to \xb_{\neg i})), g'(f(\xb))\right)\\
		& = \xnor\left(g'(f(\xb)), f'_i(\xb)\right). \qedhere
	\end{align*}
\end{proof}

\if False
	
	Let $f$ be a function taking Boolean variables, and $x$ be one of its variables. According to the intuition as previously described, our purpose is to relate an inversion of $x$ to the resulting variation of $f$. The following notations are used: 
	\begin{itemize}
		\item $\Bb$: $\{\True, \False\}$ equipped with Boolean logic,
		\item $\Db$: a discrete set such as natural or integer numbers,
		\item $\Rb$: real set.
	\end{itemize}
	
	\begin{definition}\label{def:BoolOrder}
		Order relations `$<$' and `$>$' in $\Bb$ are defined as follows:
		\begin{equation}
			\False < \True, \quad \True > \False.
		\end{equation}
	\end{definition}
	
	The differentiation in $\Rb$ is defined by introducing some variation $\delta > 0$ to the variable. In a discrete domain, although such a continuous neighborhood does not exist, it is possible to introduce a unit increase or unit decrease to the variable. However, a variable in $\Bb$ can only have one-direction variation, i.e., vary to $\False$ when being at $\True$, and vary to $\True$ when being at $\False$. Therefore, we firstly need to define the variation of a Boolean variable.
	
	\begin{definition}\label{def:BoolVariation}
		For $x \in \Bb$, a variation of $x$, denoted $\bvar{x}$, is defined as follows:
		\begin{equation}
		\bvar{x} \deq \begin{cases} 
			\True, & \textrm{if } y > x,\\
			0, & \textrm{if } y = x,\\
			\False, & \textrm{if } y < x,
		\end{cases}
		\end{equation}
		where $y$ is the resulting new value of $x$ due to this variation.
	\end{definition}
	
	\begin{definition}
		Let $\xb = (x_1, \ldots, x_n) \in \Bb^n$, $n \geq 1$. For $ 1 \le i \le n$, let $x'_i$ be the new value of $x_i$ due to a variation. Write:
		\begin{align}
			\xb'_{i} & := (x_1, \ldots, x_{i-1}, x'_i, x_{i+1}, \ldots, x_n),\\
			\bvar_i\xb & := (0, \ldots, 0, \bvar{x_i}, 0, \ldots, 0).
		\end{align}
	\end{definition}
	
	\begin{definition}\label{def:BoolFuncVar0}
		Let $f \in \Fc(\Bb^n, \Bb)$, $n \ge 1$, and $\xb \in \Bb^n$, 
		\begin{enumerate}
			\item The variation of $f$ due to the variation $\xb \to \xb' \in \Bb$, denoted $\bvar{f}(\xb \to \xb')$, is defined according to \defref{def:BoolVariation}, i.e.,
			\begin{equation}
				\bvar{f}(\xb \to \xb') \deq 
				\begin{cases} 
					\True, & \textrm{if } f(\xb') > f(\xb),\\
					0, & \textrm{if } f(\xb') = f(\xb),\\
					\False, & \textrm{if } f(\xb') < f(\xb).
				\end{cases}
			\end{equation}
			\item 
		\end{enumerate}
		For $\xb \in \Bb^n$, the variation of $f$ due to a variation of the $i$-th coordinate of $\xb$, denoted $\bvar_i{f}$, is defined according to \defref{def:BoolVariation}, i.e.,
		\begin{equation}
			\bvar_i{f} := \bvar{f}(\xb \to \xb'_i) \deq 
			\begin{cases} 
				\True, & \textrm{if } f(\xb'_i) > f(\xb),\\
				0, & \textrm{if } f(\xb'_i) = f(\xb),\\
				\False, & \textrm{if } f(\xb'_i) < f(\xb).
			\end{cases}
		\end{equation}
	\end{definition}
	
	With reference to definitions \ref{def:BoolVariation} and \ref{def:BoolFuncVar0}, we define a three-valued logic as follows.

	\begin{definition}\label{def:BoolFuncVariation}
		Let $f \colon \Bb^n \to \Bb$ and $\xb = (x_1, \ldots, x_n) \in \Bb^n$, $n \ge 1$. The variation of $f$ w.r.t a variation of the $i$-th coordinate of $\xb$, $1 \le i \le n$, denoted $\bvar_i{f}/\bvar{x_i}$, is defined as:
		\begin{equation}
			\frac{\bvar_i{f}}{\bvar{x_i}} \deq \xnor(\bvar{x_i},\bvar_i{f}).
		\end{equation}
		\if False
		given by the following truth table \ref{tab:BoolFuncVariation}:
		{\renewcommand{\arraystretch}{1.2}
		\begin{table}[!h]
			\centering
			\vspace{.2cm}
			\begin{tabular}{| c | C{2em} | C{2em} | C{2em} |}
				\toprule[1pt]
				\backslashbox{$\bvar{x_i}$}{$\bvar_i{f}$} & $\True$ & $0$ & $\False$ \\ \hline\hline 
				$\True$	& 	$ \True$ 	& $0$	& $\False$ \\ \hline 
				$0$  		& $0$ 		& $0$	& $0$ \\ \hline 
				$\False$	& $\False$ 	& $0$ 	& $\True$ \\ \bottomrule[1pt]
			\end{tabular}
			\caption{Truth table of Boolean function variation.}
			\label{tab:BoolFuncVariation}
		\end{table}}
		\fi
	\end{definition}

	\begin{proposition}\label{prop:BooleanChainrule}
		Let  $f \colon \Bb^n \to \Bb$,  $g \colon \Bb \to \Db$ where $\Db$ denotes $\Bb$, $\Nb$, $\Zb$, or $\Rb$, for $x_i \in \Bb$, $1 \leq i \leq n$:
		\begin{equation}
			\frac{\bvar{(g \circ f)}}{\bvar{x_i}} = \xnor(\frac{\bvar{g}}{\bvar{f}}, \frac{\bvar{f}}{\bvar{x_i}}).
		\end{equation}
	\end{proposition}
	\begin{proof}
		We have $\bvar{g}/\bvar{f} = \xnor(\bvar{g}, \bvar{f})$, and $\bvar{f}/\bvar{x_i} = \xnor(\bvar{f}, \bvar{x_i})$ due to \propref{prop:FuncVariation}. Using associativity of $\xnor$, we have: 
		\begin{align*}
			\xnor(\frac{\bvar{g}}{\bvar{f}}, \frac{\bvar{f}}{\bvar{x_i}}) 
			& = \xnor(\xnor(\bvar{g}, \bvar{f}), \xnor(\bvar{f}, \bvar{x_i})), \\
			& = \xnor(\xnor(\bvar{g}, \bvar{f}, \bvar{f}), \bvar{x_i}), \\
			& = \xnor(\xnor(\bvar{g}, \xnor(\bvar{f}, \bvar{f})), \bvar{x_i}), \\
			& = \xnor(\xnor(\bvar{g}, \True), \bvar{x_i}), \\
			& = \xnor(\bvar{g}, \bvar{x_i}), \\
			& = \bvar{g}/\bvar{x_i}.
		\end{align*}
		Hence the result follows. For illustration, the proof can be also obtained using truth table, which is given in \tref{tab:Composition} for the case $g \colon \Bb \to \Bb$. 
		{\renewcommand{\arraystretch}{1.0}
			\begin{table}[!bh]
				\centering
				\caption{Truth table of function composition. ``T'', ``F'', stands for $\True$, $\False$, resp.}
				\label{tab:Composition}
				\vspace{.2cm}
				\begin{tabular}{ccccccc}
					\toprule[1pt]
					\multirow{2}{*}{$\bvar{x_i}$} & \multirow{2}{*}{$\bvar{f}$} & \multirow{2}{*}{$\bvar{g}$} & \multirow{2}{*}{$\bvar{f}/\bvar{x_i}$} & \multirow{2}{*}{$\bvar{g}/\bvar{f}$} & \multirow{2}{*}{$\bvar{g}/\bvar{x_i}$} & \multirow{2}{*}{$\xnor(\frac{\bvar{g}}{\bvar{f}}, \frac{\bvar{f}}{\bvar{x_i}})$} \\
					& & & & & & \\
					\toprule[1pt]
					T & T & T & T & T & T & T \\
					T & T & F & T & F & F & F \\
					\midrule[0.25pt]
					T & F & T & F & F & T & T \\
					T & F & F & F & T & F & F \\
					\midrule[0.5pt]
					F & T & T & F & T & F & F \\
					F & T & F & F & F & T & T \\
					\midrule[0.25pt]
					F & F & T & T & F & F & F \\
					F & F & F & T & T & T & T \\
					\bottomrule[1pt]
				\end{tabular}
			\end{table}
		}
	\end{proof}

\fi

\section{Boolean Logic BackPropagation}

With the notions introduced in \sref{sec:BoolVariation}, we can write signals involved in the backpropagation process as shown in \fref{fig:Signals}. Therein, layer $l$ is a Boolean layer of consideration. For the sake of presentation simplicity, layer $l$ is assumed a fully connected layer, and:
\begin{equation}\label{eq:FC}
	x_{k, j}^{l+1} = w_{0, j}^l + \sum_{i=1}^m \Bm\pren{x_{k, i}^l, w_{i, j}^l}, \quad 1 \leq j \leq n,
\end{equation}
where $\Bm$ is the utilized Boolean logic, $k$ denotes sample index in the batch, $m$ and $n$ are the usual layer input and output sizes. Layer $l$ is connected to layer $l+1$ that can be an activation layer, a batch normalization, an arithmetic layer, or any others. Nature of $\bvar{L}/\bvar{x_{k,j}^{l+1}}$ depends on the property of layer $l+1$. For instance, it can be a usual gradient if layer $l+1$ is a real-valued arithmetic layer, or a Boolean variation if layer $l+1$ is a Boolean layer. 
Given $\bvar{L}/\bvar{x_{k,j}^{l+1}}$, Boolean layer $l$ needs to optimize its Boolean weights and compute signal $\bvar{L}/\bvar{x_{k,i}^{l}}$ for the upstream. 

\begin{figure}[!t]
	\centering
	\includegraphics[width=0.7\columnwidth]{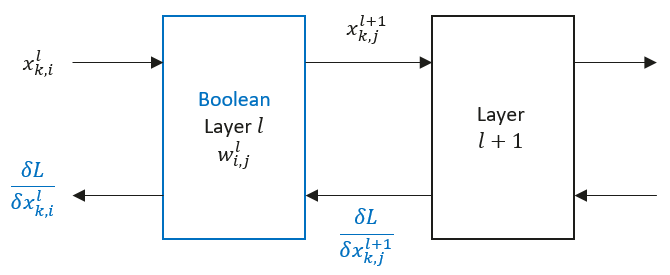}
	\caption{Illustration of signals with a Boolean linear layer. $L$ denotes the loss function.}
	\label{fig:Signals}
\end{figure}

\subsection{Optimization Logic}\label{sec:OptimLogic}

With reference to \fref{fig:Signals}, we develop logic for optimizing each weight $w_{i, j}^l$ of Boolean layer $l$. Following the intuition as previously described, we need to compute the variation of the loss w.r.t. weight $w_{i, j}^l$, i.e., $\bvar{L}/\bvar{w_{i,j}^l}$. Using \propref{prop:Chainrule}, for data sample $k$ in the batch, we can have:
\begin{equation}\label{eq:AtomicWeightVariation}
	q_{i,j,k}^l := \frac{\bvar{L}}{\bvar{w_{i,j}^{l}}}|_k = \xnor\pren{\frac{\bvar{L}}{\bvar{x_{k,j}^{l+1}}}, \frac{\bvar{x_{k,j}^{l+1}}}{\bvar{w_{i,j}^l}}},
\end{equation}
for which we need to derive $\bvar{x_{k,j}^{l+1}}/\bvar{w_{i,j}^l}$ taking into account the utilized logic $\Bm$. With reference to \eref{eq:FC}, the variation of $x_{k,j}^{l+1}$ w.r.t., $w_{i,j}^l$ as defined in \sref{sec:BoolVariation} can be obtained by simply establishing its truth table. For XOR neuron, its variation truth table is given in \tref{tab:AtomicWeightXOR} from which we can obtain:
\begin{equation}\label{eq:VarXOR}
	\textrm{XOR neuron: } \frac{\bvar{x_{k,j}^{l+1}}}{\bvar{w_{i,j}^l}} = \neg x_{k,i}^l.
\end{equation}
Following the same method, one can obtain the result for any other Boolean logic. In particular, for XNOR neuron, given that $\xnor = \neg \xor$, we can directly obtain the following: 
\begin{equation}\label{eq:VarXNOR}
	\textrm{XNOR neuron: } \frac{\bvar{x_{k,j}^{l+1}}}{\bvar{w_{i,j}^l}} = x_{k,i}^l.
\end{equation}
Hence, per-sample variation $q_{i,j,k}^l$ of weight $w_{i,j}^l$ can be obtained for any utilized logic $\Bm$. 

We now aggregate it over the batch dimension. Due to \eref{eq:AtomicWeightVariation}, datatype of $\bvar{L}/\bvar{x_{k,j}^{l+1}}$ decides the that of $q_{i,j,k}^l$, i.e., Boolean or not. In the case that $\bvar{L}/\bvar{x_{k,j}^{l+1}}$ is non Boolean, so is $q_{i,j,k}^l$, and the aggregation is performed by real summation. In the case that $q_{i,j,k}^l$ is Boolean, the aggregation can be based on majority rule, i.e., the difference between the number of True and of False. Both cases can be formulated as follows.
\begin{definition}\label{def:RealBoolValuation}
	Let $\one{\cdot}$ be the indicator function. For $b \in \Bb$ and variable $x$, define:
	\begin{equation}
		\one{x = b} = \begin{cases}
			1, & \textrm{if } x_{\bool} = b,\\
			0, & \textrm{otherwise}.
		\end{cases}
	\end{equation}
\end{definition}
The aggregation of weight variation is given as:
\begin{equation}\label{eq:AggrWeightVariation}
	q_{i,j}^l := \frac{\bvar{L}}{\bvar{w_{i,j}^{l}}} = \sum_k \one{q_{i,j,k}^l = \True}|q_{i,j,k}^l| - \sum_k \one{q_{i,j,k}^l = \False}|q_{i,j,k}^l|.
\end{equation}
Given this information, the rule for inverting $w_{i,j}^{l}$ subjected to decreasing the loss is simply given according to its definition as:
\begin{equation}\label{eq:OptimLogic}
	\boxed{w_{i,j}^l = \neg w_{i,j}^l  \textrm{ if } \xnor\pren{q_{i,j}^l, w_{i,j}^l} = \True.}
\end{equation}  
It can be also simply given by truth table \tref{tab:OptimLogic}. Equation \ref{eq:OptimLogic} is the core optimization logic based on which more sophisticated forms of optimizer can be developed in the same manner as different methods such as Adam have been developed from the basic gradient descent principle. For instance, the following is one optimizer that accumulates $q_{i,j}^l$ over training iterations. Denote by $q_{i,j}^{l,t}$ the optimization signal at iteration $t$, and by $m_{i,j}^{l,t}$ its accumulator with $m_{i,j}^{l,0} := 0$ and: 
\begin{equation}\label{eq:Accum}
	m_{i,j}^{l,t+1} = \beta^t m_{i,j}^{l,t} + \eta^t q_{i,j}^{l,t+1},
\end{equation}
where $\eta^t$ is an accumulation factor that can be tuned as a hyper-parameter, and $\beta^t$ is an auto-regularizing factor that expresses the system's state at time $t$. Its usage is linked to brain plasticity \citep{Fuchs2014} and Hebbian theory \citep{Hebb2005}, stating that ``cells that fire together wire together'', forcing weights to adapt to their neighborhood during the learning process. For the chosen weight's neighborhood, for instance, neuron, layer, or network level, $\beta^t$ is given as:
\begin{equation}
	\beta^t = \frac{\textrm{Nb of unchanged weights at } t}{\textrm{Total number of weights}}.
\end{equation}
In the experiments presented later, $\beta^t$ is set to per-layer basis and initialized as $\beta^0 = 1$. Finally, the accumulation-based optimizer is described in \aref{algo:AccumOptim}.

{\renewcommand{\arraystretch}{1.0}
	\begin{table}[!t]
		\centering
		\caption{Variation truth table of XOR neuron.}
		\label{tab:AtomicWeightXOR}
		\begin{tabular}{cccccc}
			\toprule 
			\multirow{2}{*}{$x_{k,i}^l$} 
			& \multirow{2}{*}{$w_{i,j}^l$} 
			& \multirow{2}{*}{$\neg w_{i,j}^l$} 
			& \multirow{2}{*}{$\bvar{w_{i,j}^l}$} 
			& \multirow{2}{*}{$\bvar{x_{k,j}^{l+1}}$} 
			& \multirow{2}{*}{$\bvar{x_{k,j}^{l+1}}/\bvar{w_{i,j}^l}$}   \\
			& & & & & \\
			\midrule 
			$\True$  &  $\True$  &  $\False$  & $\False$  &  $\True$  & $\False$  \\
			
			$\True$  &  $\False$ &  $\True$   & $\True$   &  $\False$ & $\False$  \\
			
			$\False$ &  $\True$  &  $\False$  & $\False$  &  $\False$ & $\True$   \\
			
			$\False$ &  $\False$ &  $\True$   & $\True$   &  $\True$  & $\True$   \\
			\bottomrule
		\end{tabular}
	\end{table}
}

{\renewcommand{\arraystretch}{1.0}
	\begin{table}[!t]
		\centering
		\caption{Optimization logic truth table.}
		\label{tab:OptimLogic}
		\begin{tabular}{ccc}
			\toprule 
			\multirow{2}{*}{$\bvar{L}/\bvar{w_{i,j}^{l}}$} 
			& \multirow{2}{*}{$w_{i,j}^l$} 
			& \multirow{2}{*}{Action} \\
			& & \\
			\midrule 
			$\True$  &  $\True$  &  Invert \\
			
			$\True$  &  $\False$ &  Keep   \\
			
			$\False$ &  $\True$  &  Keep   \\
			
			$\False$ &  $\False$ &  Invert \\
			\bottomrule
		\end{tabular}
	\end{table}
}

\begin{algorithm}[!t]
	\SetKwInOut{Input}{Input}
	\SetKwInOut{Output}{Output}
	\SetKwBlock{Loop}{Loop}{end}
	\SetKwBlock{Initialize}{Init}{end}
	\SetKwFor{When}{When}{do}{end}
	\SetKwFunction{Wait}{Wait}
	\SetAlgoLined
	\Input{
		$\eta^0$ --accumulation rate, $T$ --number of iterations\;
	}
	\Initialize{	
		$\beta^0 = 1$\;
		$m_{i,j}^{l,0} = 0$\;
	}
	\For{$t=0,\dots, T-1$}{
		Compute $q_{i,j}^{l,t+1}$ \;
		Update $m_{i,j}^{l,t+1}$ \;
		\eIf{$\xnor\pren{m_{i,j}^{l,t+1}, w_{i,j}^{l,t}} = \True$}{
			$w_{i,j}^{l,t+1} \gets \neg w_{i,j}^{l,t}$\;
			$m_{i,j}^{l,t+1} \gets 0$\;
		}{
			$w_{i,j}^{l,t+1} \gets w_{i,j}^{l,t}$\;
		}
		Update $\eta^{t+1}$\;
		Update $\beta^{t+1}$\;
	}
	\caption{Accumulate optimizer}
	\label{algo:AccumOptim}
\end{algorithm}

\subsection{Backpropagation Logic}\label{sec:BackpropLogic}

In this section, we compute the backpropagation signal $\bvar{L}/\bvar{x_{k,i}^{l}}$ of the Boolean layer $l$. Following \propref{prop:Chainrule}, for each received signal $\bvar{L}/\bvar{x_{k,j}^{l+1}}$, we have:
\begin{equation}\label{eq:AtomicBprop}
	g_{k,i,j}^l := \frac{\bvar{L}}{\bvar{x_{k,i}^{l}}}|_j = \xnor\pren{\frac{\bvar{L}}{\bvar{x_{k,j}^{l+1}}}, \frac{\bvar{x_{k,j}^{l+1}}}{\bvar{x_{k,i}^l}}},
\end{equation}
for which $\bvar{x_{k,j}^{l+1}}/\bvar{x_{k,i}^l}$ needs to be derived taking into account the utilized logic $\Bm$ as in \sref{sec:OptimLogic}. In particular, using the symmetry of Boolean logic, we can directly obtain it for XOR neuron using \eref{eq:VarXOR}, and for XNOR neuron using \eref{eq:VarXNOR} as follows:
\begin{equation}
	\frac{\bvar{x_{k,j}^{l+1}}}{\bvar{x_{k,i}^l}} = \begin{cases}
		\neg w_{i,j}^l, & \textrm{ for XOR neuron},\\
		w_{i,j}^l, & \textrm{ for XNOR neuron}.
	\end{cases}
\end{equation}

Aggregating $g_{k,i,j}^l$ over the output dimension follows the same majority rule as described in \sref{sec:OptimLogic}. Hence,
\begin{equation}\label{eq:AggrBprop}
	\boxed{\frac{\bvar{L}}{\bvar{x_{k,i}^{l}}} = \sum_j \one{g_{k,i,j}^l = \True}|g_{k,i,j}^l| - \sum_j \one{g_{k,i,j}^l = \False}|g_{k,i,j}^l|.}
\end{equation}

\subsection{Preliminary Assessment}

\tref{tab:Cifar} shows a case study of this concept on CIFAR-10 with VGG-Small architecture. Therein, Boolean method is tested with 2 architectures: one without using batch normalization, and the other using batch normalization and the activation from \cite{Liu2019}.

\begin{table}[!h]
	\centering
	\caption{Experimental results on CIFAR-10 with VGG-Small.}
	\vspace{.2cm}
		\begin{tabular}{ c  c  c  c} 
			\toprule
			\textbf{Method} 
			& \textbf{W/A} 
			& \textbf{Training} 
			& \textbf{Test Accu (\%)} \\
			\midrule\midrule
			VGG-Small (FP) \cite{Zhang2018} 		& 32/32   	& Latent      	& 93.80 \\
			BinaryConnect \cite{Courbariaux2015}	& 1/32 		& Latent      	& 90.10 \\
			XNOR-Net \cite{Rastegari2016}   		& 1/1     	& Latent      	& 89.83 \\
			BNN \cite{Hubara2017}					& 1/1    	& Latent  		& 89.85 \\
			Boolean Logic w/o BN (Ours)         	& 1/1       & Boolean      	& 90.29 \\		
			Boolean Logic + BN (Ours)            	& 1/1       & Boolean  		& \textbf{92.37} \\
			\bottomrule
	\end{tabular}
	\label{tab:Cifar}
\end{table}

\section{Conclusion}

The notion of Boolean variation is introduced at the first time based on which Boolean logic backpropagation is proposed, replacing the gradient backpropagation. This is a new mathematical principle allowing for building and training deep models directly in Boolean domain with Boolean logic. Future investigations include evaluating performance and complexity gains in deep learning tasks, and applying it to different architectures including transformers.

\bibliographystyle{abbrvnat}
\bibliography{IEEEabrv,bibliography.bib}



\if False

Let $f$ be a function taking Boolean variables, and $x$ be one of its variables. According to the intuition as previously described, our purpose is to relate an inversion of $x$ to the resulting variation of $f$. The following notations are used: 
\begin{itemize}
	\item $\Bb$: $\{\True, \False\}$ equipped with Boolean logic,
	\item $\Db$: a discrete set such as natural or integer numbers,
	\item $\Rb$: real set.
\end{itemize}

\begin{definition}\label{def:BoolOrder}
	Order relations `$<$' and `$>$' in $\Bb$ are defined as follows:
	\begin{equation}
		\False < \True, \quad \True > \False.
	\end{equation}
\end{definition}

The differentiation in $\Rb$ is defined by introducing some variation $\delta > 0$ to the variable. In a discrete domain, although such a continuous neighborhood does not exist, it is possible to introduce a unit increase or unit decrease to the variable. However, a variable in $\Bb$ can only have one-direction variation, i.e., vary to $\False$ when being at $\True$, and vary to $\True$ when being at $\False$. Therefore, we firstly need to define the variation of a Boolean variable.

\begin{definition}\label{def:BoolVariation}
	For $x \in \Bb$, a variation of $x$, denoted $\bvar{x}$, is defined as follows:
	\begin{equation}
		\bvar{x} \deq \begin{cases} 
			\True, & \textrm{if } x' > x,\\
			0, & \textrm{if } x' = x,\\
			\False, & \textrm{if } x' < x,
		\end{cases}
	\end{equation}
	where $x'$ is the resulting new value of $x$ due to this variation.
\end{definition}
	
	\begin{definition}
		Let $\xb = (x_1, \ldots, x_n) \in \Bb^n$, $n \geq 1$. For $ 1 \le i \le n$, let $x'_i$ be the new value of $x_i$ due to a variation. Write:
		\begin{align}
			\xb'_{i} & := (x_1, \ldots, x_{i-1}, x'_i, x_{i+1}, \ldots, x_n),\\
			\bvar_i\xb & := (0, \ldots, 0, \bvar{x_i}, 0, \ldots, 0).
		\end{align}
	\end{definition}
	
	\begin{definition}\label{def:BoolFuncVar0}
		Let $f \in \Fc(\Bb^n, \Bb)$, $n \ge 1$, and $\xb \in \Bb^n$, 
		\begin{enumerate}
			\item The variation of $f$ due to the variation $\xb \to \xb' \in \Bb$, denoted $\bvar{f}(\xb \to \xb')$, is defined according to \defref{def:BoolVariation}, i.e.,
			\begin{equation}
				\bvar{f}(\xb \to \xb') \deq 
				\begin{cases} 
					\True, & \textrm{if } f(\xb') > f(\xb),\\
					0, & \textrm{if } f(\xb') = f(\xb),\\
					\False, & \textrm{if } f(\xb') < f(\xb).
				\end{cases}
			\end{equation}
			\item 
		\end{enumerate}
		For $\xb \in \Bb^n$, the variation of $f$ due to a variation of the $i$-th coordinate of $\xb$, denoted $\bvar_i{f}$, is defined according to \defref{def:BoolVariation}, i.e.,
		\begin{equation}
			\bvar_i{f} := \bvar{f}(\xb \to \xb'_i) \deq 
			\begin{cases} 
				\True, & \textrm{if } f(\xb'_i) > f(\xb),\\
				0, & \textrm{if } f(\xb'_i) = f(\xb),\\
				\False, & \textrm{if } f(\xb'_i) < f(\xb).
			\end{cases}
		\end{equation}
	\end{definition}
	
	With reference to definitions \ref{def:BoolVariation} and \ref{def:BoolFuncVar0}, we define a three-valued logic as follows.
	
	\begin{definition}
		Define $\Mb$ to be a three-value logic given as $\Mb \deq \{\Bb, 0\}$ with logic connectives given in \tref{tab:ThreeValueLogic}.
		{\renewcommand{\arraystretch}{1.2}
			\begin{table}[!h]
				\if False
				\begin{subtable}[c]{0.1\textwidth}
					\centering
					\begin{tabular}{| c | c |}
						\hline
						$a$  	&  	$\neg a$ 	\\ \hline 
						$\True$	& 	$ \False$ 	\\ \hline 
						$0$  		& $0$ 		\\ \hline 
						$\False$	& $\True$ 	\\ \hline
					\end{tabular}
				\end{subtable}%
				\fi
				\begin{subtable}[c]{\textwidth}
					\centering
					\begin{tabular}{| c | c |}
						\hline
						$a$  	&  	$\neg a$ 	\\ \hline 
						$\True$	& 	$ \False$ 	\\ \hline 
						$0$  		& $0$ 		\\ \hline 
						$\False$	& $\True$ 	\\ \hline
					\end{tabular}
					\subcaption{Negation}
				\end{subtable}
				\begin{subtable}{0.25\textwidth}
					\centering
					\vspace{0.5cm}
					\begin{tabular}{| c | c | c | c | c |}
						\hline
						\multicolumn{2}{|c|}{\multirow{2}{*}{$\andd$}} & \multicolumn{3}{c|}{$b$} 	\\ \cline{3-5} 
						\multicolumn{2}{|c|}{}				& $\True$	& $0$	& $\False$ 	\\ \hline
						\multirow{3}{*}{$a$}	& $\True$	& $\True$ 	& $0$	& $\False$	\\ \cline{2-5}
						& $0$		& $0$  		& $0$	& $0$		\\ \cline{2-5}
						& $\False$	& $\False$  & $0$	& $\False$	\\ \hline
					\end{tabular}
					\subcaption{AND}
				\end{subtable}%
				\begin{subtable}{0.25\textwidth}
					\centering
					\vspace{0.5cm}
					\begin{tabular}{| c | c | c | c | c |}
						\hline
						\multicolumn{2}{|c|}{\multirow{2}{*}{$\orr$}} & \multicolumn{3}{c|}{$b$} 	\\ \cline{3-5} 
						\multicolumn{2}{|c|}{}				& $\True$	& $0$	& $\False$ 	\\ \hline
						\multirow{3}{*}{$a$}	& $\True$	& $\True$ 	& $0$	& $\True$	\\ \cline{2-5}
						& $0$		& $0$  		& $0$	& $0$		\\ \cline{2-5}
						& $\False$	& $\True$  	& $0$	& $\False$	\\ \hline
					\end{tabular}
					\subcaption{OR}			
				\end{subtable}%
				\begin{subtable}{0.25\textwidth}
					\centering
					\vspace{0.5cm}
					\begin{tabular}{| c | c | c | c | c |}
						\hline
						\multicolumn{2}{|c|}{\multirow{2}{*}{$\xor$}} & \multicolumn{3}{c|}{$b$} 	\\ \cline{3-5} 
						\multicolumn{2}{|c|}{}				& $\True$	& $0$	& $\False$ 	\\ \hline
						\multirow{3}{*}{$a$}	& $\True$	& $\False$ 	& $0$	& $\True$	\\ \cline{2-5}
						& $0$		& $0$  		& $0$	& $0$		\\ \cline{2-5}
						& $\False$	& $\True$  	& $0$	& $\False$	\\ \hline
					\end{tabular}
					\subcaption{XOR}			
				\end{subtable}%
				\begin{subtable}{0.25\textwidth}
					\centering
					\vspace{0.5cm}
					\begin{tabular}{| c | c | c | c | c |}
						\hline
						\multicolumn{2}{|c|}{\multirow{2}{*}{$\xnor$}} & \multicolumn{3}{c|}{$b$} 	\\ \cline{3-5} 
						\multicolumn{2}{|c|}{}				& $\True$	& $0$	& $\False$ 	\\ \hline
						\multirow{3}{*}{$a$}	& $\True$	& $\True$ 	& $0$	& $\False$	\\ \cline{2-5}
						& $0$		& $0$  		& $0$	& $0$		\\ \cline{2-5}
						& $\False$	& $\False$ 	& $0$	& $\True$	\\ \hline
					\end{tabular}
					\subcaption{XNOR}			
				\end{subtable}
				\caption{Logic connectives of the three-value logic $\Mb$.}
				\label{tab:ThreeValueLogic}
		\end{table}}
	\end{definition}
	
	$\Mb$ can be seen as an extension of the Boolean logic by adding element $0$ whose meaning is ``ignored''. Its logic connectives are most similar to those of Bochvar's internal three-valued logic.
	
	\begin{definition}\label{def:BoolFuncVariation}
		Let $f \colon \Bb^n \to \Bb$ and $\xb = (x_1, \ldots, x_n) \in \Bb^n$, $n \ge 1$. The variation of $f$ w.r.t a variation of the $i$-th coordinate of $\xb$, $1 \le i \le n$, denoted $\bvar_i{f}/\bvar{x_i}$, is defined as:
		\begin{equation}
			\frac{\bvar_i{f}}{\bvar{x_i}} \deq \xnor(\bvar{x_i},\bvar_i{f}).
		\end{equation}
		\if False
		given by the following truth table \ref{tab:BoolFuncVariation}:
		{\renewcommand{\arraystretch}{1.2}
			\begin{table}[!h]
				\centering
				\vspace{.2cm}
				\begin{tabular}{| c | C{2em} | C{2em} | C{2em} |}
					\toprule[1pt]
					\backslashbox{$\bvar{x_i}$}{$\bvar_i{f}$} & $\True$ & $0$ & $\False$ \\ \hline\hline 
					$\True$	& 	$ \True$ 	& $0$	& $\False$ \\ \hline 
					$0$  		& $0$ 		& $0$	& $0$ \\ \hline 
					$\False$	& $\False$ 	& $0$ 	& $\True$ \\ \bottomrule[1pt]
				\end{tabular}
				\caption{Truth table of Boolean function variation.}
				\label{tab:BoolFuncVariation}
		\end{table}}
		\fi
	\end{definition}

	\begin{proposition}\label{prop:BooleanChainrule}
		Let  $f \colon \Bb^n \to \Bb$,  $g \colon \Bb \to \Db$ where $\Db$ denotes $\Bb$, $\Nb$, $\Zb$, or $\Rb$, for $x_i \in \Bb$, $1 \leq i \leq n$:
		\begin{equation}
			\frac{\bvar{(g \circ f)}}{\bvar{x_i}} = \xnor(\frac{\bvar{g}}{\bvar{f}}, \frac{\bvar{f}}{\bvar{x_i}}).
		\end{equation}
	\end{proposition}
	\begin{proof}
		We have $\bvar{g}/\bvar{f} = \xnor(\bvar{g}, \bvar{f})$, and $\bvar{f}/\bvar{x_i} = \xnor(\bvar{f}, \bvar{x_i})$ due to \propref{prop:FuncVariation}. Using associativity of $\xnor$, we have: 
		\begin{align*}
			\xnor(\frac{\bvar{g}}{\bvar{f}}, \frac{\bvar{f}}{\bvar{x_i}}) 
			& = \xnor(\xnor(\bvar{g}, \bvar{f}), \xnor(\bvar{f}, \bvar{x_i})), \\
			& = \xnor(\xnor(\bvar{g}, \bvar{f}, \bvar{f}), \bvar{x_i}), \\
			& = \xnor(\xnor(\bvar{g}, \xnor(\bvar{f}, \bvar{f})), \bvar{x_i}), \\
			& = \xnor(\xnor(\bvar{g}, \True), \bvar{x_i}), \\
			& = \xnor(\bvar{g}, \bvar{x_i}), \\
			& = \bvar{g}/\bvar{x_i}.
		\end{align*}
		Hence the result follows. For illustration, the proof can be also obtained using truth table, which is given in \tref{tab:Composition} for the case $g \colon \Bb \to \Bb$. 
		{\renewcommand{\arraystretch}{1.0}
			\begin{table}[!bh]
				\centering
				\caption{Truth table of function composition. ``T'', ``F'', stands for $\True$, $\False$, resp.}
				\label{tab:Composition}
				\vspace{.2cm}
				\begin{tabular}{ccccccc}
					\toprule[1pt]
					\multirow{2}{*}{$\bvar{x_i}$} & \multirow{2}{*}{$\bvar{f}$} & \multirow{2}{*}{$\bvar{g}$} & \multirow{2}{*}{$\bvar{f}/\bvar{x_i}$} & \multirow{2}{*}{$\bvar{g}/\bvar{f}$} & \multirow{2}{*}{$\bvar{g}/\bvar{x_i}$} & \multirow{2}{*}{$\xnor(\frac{\bvar{g}}{\bvar{f}}, \frac{\bvar{f}}{\bvar{x_i}})$} \\
					& & & & & & \\
					\toprule[1pt]
					T & T & T & T & T & T & T \\
					T & T & F & T & F & F & F \\
					\midrule[0.25pt]
					T & F & T & F & F & T & T \\
					T & F & F & F & T & F & F \\
					\midrule[0.5pt]
					F & T & T & F & T & F & F \\
					F & T & F & F & F & T & T \\
					\midrule[0.25pt]
					F & F & T & T & F & F & F \\
					F & F & F & T & T & T & T \\
					\bottomrule[1pt]
				\end{tabular}
			\end{table}
		}
	\end{proof}
	
	
	==== Old to be revised ==== 
	
	\begin{notation}
		Let $\xb = (x_1, \ldots, x_n) \in \Bb^n$ and $f \colon \Bb^n \to \Bb$, for $n \geq 1$. For the sake of notational simplicity, the following shorthand notations are used:
		\begin{itemize}
			\item $\xb'_{i} := (x_1, \ldots, x_{i-1}, x'_i, x_{i+1}, \ldots, x_n)$.
			\item $f(\neg x_i) := f(x_1, \ldots, x_{i-1}, \neg x_i, x_{i+1}, \ldots, x_n)$, for $1 \leq i \leq n$.
			\item $\bvar{f(x_i)} := \bvar{f(\xb \to \xb_i)}$, for $1 \leq i \leq n$, where $\xb_{i} := (x_1, \ldots, x_{i-1}, \neg x_i, x_{i+1}, \ldots, x_n)$.
		\end{itemize} 
	\end{notation}

	\begin{definition}\label{def:BoolFuncVariationOld}
		For $f \colon \Bb^n \to \Bb$, 
		the variation of $f$ w.r.t. $x_i \in \Bb$, for $1 \leq i \leq n$, denoted $\bvar{f}/\bvar{x_i}$, is defined as follows:
		\begin{equation}
			\frac{\bvar{f}}{\bvar{x_i}} \deq 
			\begin{cases} 
				\True, & \textrm{if } \bvar{f} = \bvar{x_i},\\
				\False, & \textrm{otherwise}.
			\end{cases}
		\end{equation}
	\end{definition}
	
	Intuitively, the variation of a function $f$ w.r.t. a Boolean input $x_i$ is defined as $\True$ if they vary in the same direction, and as $\False$ otherwise. 
	
	\begin{definition}\label{def:Real2Bool}
		For $x \in \{\Nb, \Zb, \Rb\}$, denote by $x_{\bool}$ its Boolean value defined as $x_{\bool} = \True \Leftrightarrow x \ge 0$, and $x_{\bool} = \False \Leftrightarrow x < 0$, and denote by $|x|$ its magnitude.
	\end{definition}
	
	Hence, non-Boolean variable $x$ can be represented by the couple of logic $x_{\bool}$ and magnitude $|x|$. Using this, we can define the variation of a real-valued function w.r.t. a Boolean input.
	
	\begin{definition}\label{def:RealFuncVariation}
		For $f \colon \Bb^n \to \Db$ where $\Db$ denotes $\Nb$, $\Zb$, or $\Rb$, the variation of $f$ w.r.t. $x_i \in \Bb$, for $1 \leq i \leq n$, denoted $\bvar{f}/\bvar{x_i}$, is defined as follows:
		\begin{equation}
			\frac{\bvar{f}}{\bvar{x_i}} \deq y \in \Db, \textrm{ s.t. } |y| = |\bvar{f}|, \textrm{ and } y_{\bool} = \begin{cases} 
				\True, & \textrm{if } \bvar{f}_{\bool} = \bvar{x_i},\\
				\False, & \textrm{otherwise},
			\end{cases}
		\end{equation}
		where $\bvar{f}$ is the variation of $f$ in $\Db$. 
	\end{definition}
	
	\begin{remark}\label{rk:noderivatives}
		\defref{def:BoolFuncVariation} and \defref{def:RealFuncVariation} do not define a Boolean derivative, e.g., the one by \citep{Akers1959, Thayse1973}, knowing that ``there is no `good' Boolean derivative satisfying all the three basic derivative-like properties: additivity, homogeneity, and the Leibniz rule.'' \citep{Rudeanu2009}.
	\end{remark}
	
	As $\xnor(a,b) = \True \Leftrightarrow a = b$, the following property is direct from definitions \ref{def:BoolFuncVariation} and \ref{def:RealFuncVariation}.
	
	\begin{proposition}\label{prop:FuncVariation}
		With the same notation as above, $\bvar{f}/\bvar{x_i} = \xnor(\bvar{f}, \bvar{x_i})$.
	\end{proposition}

	\begin{proposition}\label{prop:Chainrule}
		Let  $f \colon \Bb^n \to \Bb$,  $g \colon \Bb \to \Db$ where $\Db$ denotes $\Bb$, $\Nb$, $\Zb$, or $\Rb$, for $x_i \in \Bb$, $1 \leq i \leq n$:
		\begin{equation}
			\frac{\bvar{(g \circ f)}}{\bvar{x_i}} = \xnor(\frac{\bvar{g}}{\bvar{f}}, \frac{\bvar{f}}{\bvar{x_i}}).
		\end{equation}
	\end{proposition}
	\begin{proof}
		We have $\bvar{g}/\bvar{f} = \xnor(\bvar{g}, \bvar{f})$, and $\bvar{f}/\bvar{x_i} = \xnor(\bvar{f}, \bvar{x_i})$ due to \propref{prop:FuncVariation}. Using associativity of $\xnor$, we have: 
		\begin{align*}
			\xnor(\frac{\bvar{g}}{\bvar{f}}, \frac{\bvar{f}}{\bvar{x_i}}) 
			& = \xnor(\xnor(\bvar{g}, \bvar{f}), \xnor(\bvar{f}, \bvar{x_i})), \\
			& = \xnor(\xnor(\bvar{g}, \bvar{f}, \bvar{f}), \bvar{x_i}), \\
			& = \xnor(\xnor(\bvar{g}, \xnor(\bvar{f}, \bvar{f})), \bvar{x_i}), \\
			& = \xnor(\xnor(\bvar{g}, \True), \bvar{x_i}), \\
			& = \xnor(\bvar{g}, \bvar{x_i}), \\
			& = \bvar{g}/\bvar{x_i}.
		\end{align*}
		Hence the result follows. For illustration, the proof can be also obtained using truth table, which is given in \tref{tab:Composition} for the case $g \colon \Bb \to \Bb$. 
		{\renewcommand{\arraystretch}{1.0}
			\begin{table}[!bh]
				\centering
				\caption{Truth table of function composition. ``T'', ``F'', stands for $\True$, $\False$, resp.}
				\label{tab:Composition}
				\vspace{.2cm}
				\begin{tabular}{ccccccc}
					\toprule[1pt]
					\multirow{2}{*}{$\bvar{x_i}$} & \multirow{2}{*}{$\bvar{f}$} & \multirow{2}{*}{$\bvar{g}$} & \multirow{2}{*}{$\bvar{f}/\bvar{x_i}$} & \multirow{2}{*}{$\bvar{g}/\bvar{f}$} & \multirow{2}{*}{$\bvar{g}/\bvar{x_i}$} & \multirow{2}{*}{$\xnor(\frac{\bvar{g}}{\bvar{f}}, \frac{\bvar{f}}{\bvar{x_i}})$} \\
					& & & & & & \\
					\toprule[1pt]
					T & T & T & T & T & T & T \\
					T & T & F & T & F & F & F \\
					\midrule[0.25pt]
					T & F & T & F & F & T & T \\
					T & F & F & F & T & F & F \\
					\midrule[0.5pt]
					F & T & T & F & T & F & F \\
					F & T & F & F & F & T & T \\
					\midrule[0.25pt]
					F & F & T & T & F & F & F \\
					F & F & F & T & T & T & T \\
					\bottomrule[1pt]
				\end{tabular}
			\end{table}
		}
	\end{proof}

\fi

\end{document}